\documentclass[lettersize,journal]{IEEEtran}
\usepackage{tikz}
\usepackage{amsmath}

\usepackage{filecontents}

\usepackage{graphicx}
\usepackage{subfigure}
\usepackage{amsmath}
\usepackage{amsthm}
\usepackage{mathrsfs}
\usepackage{float}
\usepackage[vlined,ruled,linesnumbered,algo2e]{algorithm2e}
\usepackage{xcolor}
\usepackage{multicol}
\usepackage{multirow}
\usepackage{diagbox}
\usepackage{bbm}
\usepackage{makecell}
\usepackage{enumitem}
\usepackage{booktabs}
\usepackage{algorithm}
\usepackage{algorithmic}
\usepackage{hyperref}
\usepackage{amsfonts}
\usepackage{newtxmath}
\usepackage{dsfont}

\newtheorem{thm}{Theorem}
\newtheorem{prop}{Proposition}
\newtheorem{lem}{Lemma}
\newtheorem{cor}{Corollary}

\AtBeginDocument{%
 }

\begin{document}

\title{Auditing Machine Unlearning: A Systematic Research on Whether Models Truly Forget}

\author{Dayong Ye, Tianqing Zhu$^*$, Ruiding Huang, Xinbo Fu, Jiayang Li, Bo Liu, Huan Huo, Wanlei Zhou
\thanks{Tianqing Zhu is the corresponding author.}}

\maketitle

\begin{abstract}
Machine unlearning has been extensively studied in response to growing privacy concerns and regulatory requirements. However, auditing whether unlearning algorithms have truly erased the influence of specific data remains an open challenge. The lack of reliable and practical auditing mechanisms can lead to critical privacy risks, such as residual information leakage.
This paper initiates a systematic investigation into whether existing unlearning algorithms can truly forget the designated data. We propose the first practical and general-purpose auditing framework for machine unlearning, inspired by the concept of proof of ignorance. Our framework addresses the key practicality limitations of existing methods by eliminating the need for retraining-from-scratch baselines, avoiding the training of large numbers of shadow models, and requiring no intrusive intervention in the original training process.
To evaluate the effectiveness of our framework, we first conduct validation experiments to verify its soundness and completeness. We then perform comprehensive experiments across six datasets and ten representative unlearning methods. The results demonstrate that our framework reliably distinguishes between successful and failed unlearning. In particular, we observe that retraining-based and fine-tuning-based methods can achieve effective unlearning, even when the target data remain in the original dataset. In contrast, de-optimization-based methods fail to achieve true unlearning and instead degrade the model’s performance. Fisher/Hessian-based methods also fail to unlearn requested data, even formal certification is provided.
Moreover, we show that our framework is robust against fake unlearning attempts and generalizes well to large language models. 
\end{abstract}




\pagestyle{plain}

\section{Introduction}
Machine unlearning has emerged as a critical paradigm for enabling data deletion in machine learning models, particularly in response to growing concerns around user privacy and regulatory requirements such as the General Data Protection Regulation (GDPR) \cite{GDPR}. The goal of unlearning is to ensure that a trained model forgets specific data samples upon request, such that the resulting model behaves as if the data had never been used in training \cite{Cao15,Bourtoule21}. Motivated by this objective, a wide range of unlearning algorithms has been proposed \cite{Xu23,Ye25USENIX}. However, the problem of auditing unlearning outcomes remains largely underexplored. In most existing studies, outcome auditing is treated as a peripheral issue or mentioned only briefly, without a systematic investigation \cite{Thudi23,Zhang24ICML,Xue25}.

Auditing machine unlearning is critical. Its goal is to determine whether the unlearned model has truly forgotten the knowledge associated with the data requested for removal. An effective audit not only reassures data owners who exercise their right to revoke personal data, but also certifies the correctness and reliability of unlearning algorithms.  
Straightforward auditing approaches include parametric auditing and behavioral auditing \cite{Thudi23,Xue25}. Parametric auditing evaluates whether the parameters of the unlearned model remain within a bounded distance from those of a model retrained from scratch without the unlearning data. In contrast, behavioral auditing assesses if the unlearned model exhibits similar behavior to that of the retrained model when evaluated on specific datasets. 
Despite their conceptual soundness, both approaches share a fundamental limitation: they require access to a retrained model that excludes the unlearning data. This renders them impractical for two reasons. First, retraining a model from scratch incurs substantial computational overhead.  
Second, if one already possesses the retrained model, auditing becomes unnecessary, as this retrained model can be directly deployed as the unlearned model. 

Another line of research on unlearning auditing is based on membership inference \cite{Naderloui25USENIX,Hayes25SaTML}, which aims to determine whether the unlearned data still belong to the training set of the unlearned model. However, membership inference typically requires training a large number of shadow models, resulting in heavy computational cost \cite{Carlini22SP}. Moreover, it often suffers from limited inference accuracy and yields only probabilistic outcomes without explicit accuracy guarantees \cite{Zhang25SaTML,Chowdhury25SaTML,Wang25CCS}. 
Other less mainstream auditing methods rely on backdoors, but they require the unrealistic step of injecting backdoors into the training data before model training \cite{Sommer22PETS}. 

An orthogonal research direction focuses on auditing unlearning procedures \cite{Zhang24ICMLb,Koloskova25ICML,Eisenhofer25}, with the goal of verifying whether mutually agreed unlearning algorithms are executed correctly. In contrast, our work audits unlearning outcomes rather than procedures. This distinction is important because, as we show in the experiments, some unlearning algorithms fail to truly remove the influence of the unlearning data even if they are executed correctly.


To overcome these limitations, we introduce a practical and general unlearning auditing framework  
by addressing \textbf{three challenges}. \textbf{First}, the absence of access to retrained reference models necessitates a self-contained auditing mechanism that does not rely on expensive retraining from scratch. \textbf{Second}, achieving a balance between auditing effectiveness and computational efficiency complicates the design: the auditing method must deliver reliable assessments while incurring significantly lower computational cost than full retraining. \textbf{Third}, there is no definitive ground truth for evaluating whether a given unlearning algorithm, apart from retraining from scratch, has truly achieved the unlearning objective. 

To address the first and second challenges, our auditing framework constructs verification models by training on only $15\%$ of the remaining data, on average, for each verification model. As our verification model is trained on data disjoint from the unlearning set, it provides a valid comparator for evaluating whether the unlearned model has successfully forgotten the targeted data, thus addressing the first challenge. This comparison is grounded in the concept of proof of ignorance. If a verification model, believed to lack knowledge of the unlearning set, can be transformed into the unlearned model through a polynomial-time process, such as updating it with additional data, then the unlearned model is also considered to be ignorant of the same knowledge. Thus, if the unlearned model exhibits behavior closely aligned with that of the verification model or its transformed version, it provides strong evidence that the unlearned model no longer retains the forgotten data.
Moreover, although only a subset of the remaining data is used, our framework leverages the insight that a small but representative subset can still capture the distributional characteristics of the original data \cite{Olson18NeurIPS}, thus addressing the second challenge by striking a balance between computational efficiency and auditing accuracy. 

To address the third challenge, the lack of ground truth for successful unlearning, we first validate our framework before using it to audit unlearning methods. Then, we use backdoor-based and membership inference-based auditing, along with t-SNE \cite{t-SNE} visualization, as supporting evidence. 
In summary, this paper makes \textbf{three contributions}.
\begin{itemize}[leftmargin=*]
    \item We initiate a systematical study to examine whether existing popular unlearning algorithms can truly forget the revoked data. To achieve this, we propose the first practical and general auditing framework, grounded in the concept of proof of ignorance. Before applying the framework, we rigorously evaluate its effectiveness through both theoretical analysis and validation experiments.
    \item We perform a comprehensive empirical study that evaluates ten unlearning approaches across six datasets, including both image and text domains. We also assess our framework against fake unlearning attempts and demonstrate its effectiveness in successfully detecting them.
    \item We further demonstrate the strong generalizability of our framework by extending the auditing study to large language model (LLM) unlearning scenarios.
\end{itemize}

\vspace{-1mm}
\section{Preliminary and Threat Model}
\vspace{-0mm}
\noindent\textbf{Machine Unlearning.} Given a learning algorithm $A:\mathcal{D}\rightarrow\mathcal{M}$, where $\mathcal{D}$ denotes the data space and $\mathcal{M}$ represents the model space, a model $M$ can be trained using $A$ on a dataset $D$, i.e., $M = A(D)$. When an unlearning request is made to forget a subset of the data $D_u \subset D$, a straightforward approach is to retrain the model from scratch using $A$ on the remaining data $D_r=D - D_u$, resulting in a perfectly unlearned model $M^*_u$, i.e., $M^*_u = A(D_r)$. However, this retraining-from-scratch method is computationally intensive. 

Let $U: \mathcal{M} \times \mathcal{D} \times \mathcal{D} \rightarrow \mathcal{M}$ be an unlearning algorithm. An unlearned model $M_u$ is obtained by applying $U$ to $M$, $D$, and $D_u$, i.e., $M_u = U(M, D, D_u)$. Ideally, $U$ is more computationally efficient than retraining from scratch, and the difference between $M_u$ and $M^*_u$ is minimized. This difference can be quantified in either the model parameter space \cite{Guo20,Graves21AAAI} or the model output space \cite{NIPSUnlearning23,Maini24,Georgiev25ICLR}. 

\vspace{1mm}
\noindent\textbf{Proof of Ignorance.} In the proof of ignorance \cite{Deshpande18,Kuykendall20TCC}, given an instance $q$ of a complex problem, the prover must demonstrate to the auditor that they do not know the solution $s$ for the problem $q$. 
To address this, the prover identifies a well-known complex instance $q'$, with a solution $s'$ that is computationally related to $s$, namely, there exists a polynomial-time process that transforms $s'$ to $s$. As $q'$ is both well-known and complex, claiming ignorance of its solution $s'$ is considered trustworthy by the auditor. Moreover, since $s'$ is computationally related to $s$, ignorance of $s'$ implies ignorance of $s$.

\vspace{1mm}
\noindent\textbf{Threat Model.} The model provider employs an unlearning algorithm $U$ to remove the influence of the requested data $D_u$, as specified by the data owner, resulting in an unlearned model $M_u$. The model provider, acting as the \emph{prover}, then aims to validate that the unlearned model $M_u$ no longer contains information related to $D_u$. This validation can be presented either to the data owner or to a trusted third party, such as a regulator, who acts as the \emph{auditor}.

The prover has full access to the original model $M$ and the training data $D$, and can freely modify $M$ or utilize $D$ as needed. 
In contrast, the auditor does not have access to the internal parameters of $M$ or its variants but can query these models and observe their outputs. The auditor holds the authority to define and enforce the auditing protocol. In particular, the auditor can specify which portions of the data $D$ are used during the auditing process without accessing the actual data content, and the prover follows these auditing instructions throughout the protocol. This assumption is reasonable, as our objective is to audit the efficacy of unlearning algorithms rather than the honesty of the prover. 
We will relax this assumption in the experiments by considering scenarios in which the prover performs fake unlearning.

This asymmetry of access reflects real-world scenarios where the model provider (the prover) is an ML service platform, while the auditor is a data owner or regulatory authority with limited insight into the internal workings of the model. The assumption that the auditor cannot see the raw data in $D$, but can still reference its structure (e.g., via hashes), is reasonable. It reflects the common setting in privacy-preserving audits \cite{Lycklama24USENIX}, where the auditor must operate under a black-box model due to data sensitivity.

\vspace{1mm}
\noindent\textbf{Problem Statement.} Given an unlearned model $M_u$, the original model $M$, the unlearning data $D_u$, and the original training data $D$, our goal is to design an auditing framework $\mathcal{V}$ that produces a binary outcome indicating whether $M_u$ has successfully forgotten $D_u$. Specifically, $\mathcal{V}(M_u, M, D_u, D) \rightarrow \{0, 1\}$, where ``$1$'' denotes successful unlearning and ``$0$'' indicates failure.
The notations used in this paper are shown in Table \ref{tab:notation}.

\vspace{-1mm}
\begin{table}[!ht]\scriptsize
	\centering
	\caption{Summary of Notations}
    \vspace{-1mm}
\begin{tabular} {cl}
\toprule
{\bf Notations} & \hspace{8em}{\bf Description} \\
\midrule
$M$ & The pre-unlearning original model  \\
$D$ & The data used to train $M$  \\
$D_u\subset D$ & The data requested to be unlearned from $M$ \\
$D_r$ & The remaining data, where $D_r\cup D_u=D$ \\
$D_t$ & The test data, where $D_t\cap D=\emptyset$\\
$M_u$ & The unlearned model claimed to have forgotten $D_u$ \\
$D_v\subset D_r$ & The data used for verification  \\
$M_v$ & The verification model trained by $D_v$\\
\bottomrule
\end{tabular}
	\label{tab:notation}
    \vspace{-2mm}
\end{table}

\section{Methodology}
\noindent\textbf{Rationale.} 
According to the proof of ignorance, the prover can generate another model $M_v$ that is guaranteed to exclude the data $D_u$. If the prover can then demonstrate that $M_v$ can be transformed into $M_u$ through a polynomial-time process, such as further updating $M_v$ with additional data from $D_r$, it becomes credible to assert that $M_u$ does not contain information from $D_u$. Detailed theoretical analyses of our auditing framework are provided in Section \ref{sec:analysis}. 


\begin{figure}[ht]
\vspace{-0mm}
\centering
    \includegraphics[scale=0.23]{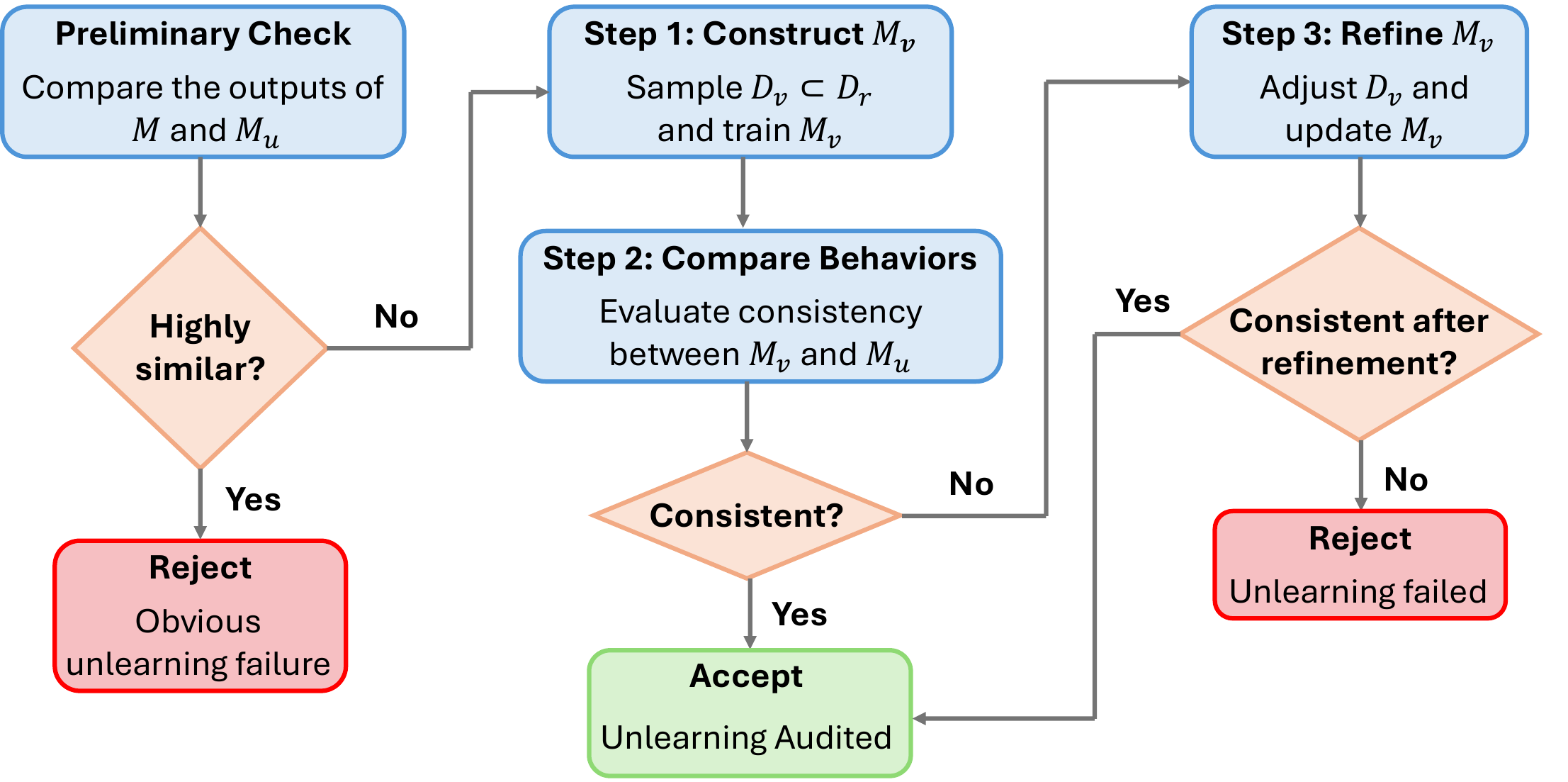}
    \vspace{-0mm}
	\caption{Overview of our auditing framework. The auditor first conducts a preliminary check by comparing outputs from the unlearned model $M_u$ and the pre-unlearning model $M$. If their behavior is highly consistent, unlearning is immediately considered failed. Otherwise, the auditor proceeds to the main protocol: Step 1, the prover trains a verification model $M_v$; Step 2, the auditor compares outputs between $M_v$ and $M_u$; Step 3, if inconsistent, $M_v$ is updated or retrained and compared again. If consistency is achieved, unlearning is accepted; otherwise, it is rejected.}
    \vspace{-1mm}
	\label{fig:Overview}
\end{figure}

\vspace{0mm}
\noindent\textbf{Overview.} As shown in Figure~\ref{fig:Overview}, before initiating the formal auditing process, we conduct a preliminary check: if the unlearned model exhibits behavior nearly identical to the pre-unlearning model, for example, both achieving high accuracy on the unlearning set, it can be immediately flagged as a failure of unlearning. This is because a truly unlearned model is expected to undergo a noticeable behavioral shift, particularly in its performance on the unlearning data. 
Once such obvious failures are filtered out, we proceed with the main auditing protocol, which consists of \textbf{three steps}. 

\textbf{First}, the auditor randomly identifies a subset of data $D_v$ from the remaining data $D_r=D-D_u$, and instructs the prover to use $D_v$ as a verification set to train a new model $M_v$ for auditing. Since $D_v$ is entirely disjoint from $D_u$, it is guaranteed that $M_v$ does not contain any information from $D_u$. 
\textbf{Second}, the auditor compares the outputs of $M_v$ and $M_u$. If the outputs of the two models are consistent, it can be trusted that $M_u$ also does not contain the information of $D_u$. This is because consistent outputs indicate that $M_u$ and $M_v$ exhibit similar decision boundaries on $D_u$. As $M_v$ is independent of $D_u$, such consistence implies that $M_u$'s behavior does not rely on specific features derived from $D_u$. However, if the outputs of the two models are inconsistent, the audit proceeds to the third step. In the \textbf{third} step, based on the observed differences in outputs between $M_v$ and $M_u$, the prover decides whether to increase the size of $D_v$ by adding more data from $D_r$ or to reduce the size of $D_v$ by removing some data. If the decision is to increase the size of $D_v$, the prover fine-tunes $M_v$ using the additional data. If the decision is to reduce the size of $D_v$, the prover retrains $M_v$ using the reduced subset. After refining $M_v$, the auditor compares the outputs of $M_v$ and $M_u$ again. If the outputs are consistent, the prover can claim that $M_u$ does not contain the information of $D_u$. If the outputs remain inconsistent, the prover fails to prove ignorance of the presence of $D_u$ in $M_u$. 
The auditor grants the prover only a single opportunity to refine $M_v$ for auditing. We will show in Section \ref{sec:analysis} that this single refinement is sufficient to ensure correctness.
Each step of our method is elaborated in detail. 

\vspace{1mm}
\noindent\textbf{Step 1. Verification Model Construction.} The auditor randomly selects a subset of data from $D_r$ to form a new dataset $D_v$. This selection can be implemented using various randomized sampling techniques to ensure that $D_v$ is representative of the overall distribution of $D_r$ while remaining entirely disjoint from $D_u$. One straightforward approach is uniform random sampling, where each data point in $D_r$ has an equal probability of being included in $D_v$. This method ensures that the selection process is unbiased and does not favor any particular subset of data.
Alternatively, the auditor can use stratified random sampling, where the data points are grouped based on their class labels. Samples are then drawn randomly from each group in proportion to their representation in $D_r$. This technique ensures that $D_v$ retains a distribution similar to that of $D_r$, making the auditing process more robust against imbalances in the dataset. 
Both approaches will be evaluated. 

For the size of $D_v$, a common heuristic is to define it as a fraction of the size of $D_r$, i.e., $|D_v| = \alpha \cdot |D_r|$, where $\alpha \in (0,1]$. A larger $\alpha$ ensures that $D_v$ is more representative of $D_r$, while a smaller $\alpha$ reduces computational overhead. We treat $\alpha$ as a hyperparameter that can be tuned based on auditing accuracy requirements.
After forming $D_v$, the prover uses it to train a verification model $M_v$. As $D_v$ is disjoint from $D_u$, it is guaranteed that $M_v$ does not incorporate any information related to $D_u$. This property makes $M_v$ a clean reference model for auditing the unlearning process.

\vspace{1mm}
\noindent\textbf{Step 2. Output Comparison for Auditing.} After training the verification model $M_v$, the auditor compares its outputs with those of the unlearned model $M_u$ using the unlearning set $D_u$. For each sample $x \in D_u$, the auditor checks if $M_v(x) = M_u(x)$ and computes an auditing metric $\upsilon$, defined as: 
\begin{equation}\label{eq:verification metric}
    \upsilon=\frac{\sum_{x\in D_u}\mathds{1}[M_v(x)=M_u(x)]}{|D_u|}, 
\end{equation}
where $\mathds{1}_{M_v(x)=M_u(x)}$ is an indicator function that equals $1$ if the outputs of $M_v$ and $M_u$ are identical for a given $x$, and $0$ otherwise. The metric $\upsilon$ represents the proportion of data samples in $D_u$ for which the two models produce consistent outputs. Ideally, $\upsilon$ should equal $1$ to provide absolute certainty in the audit. However, achieving $\upsilon = 1$ is challenging due to the arbitrary initialization and inherent variability in model training. To address this, we adopt a pre-defined threshold $\zeta$, treated as a hyperparameter. If $\upsilon$ exceeds the threshold $\zeta$, it is trusted that $M_u$ does not contain information from $D_u$. Otherwise, the audit proceeds to the third step.

The threshold $\zeta$ determines the acceptable level of consistency. 
It should be based on the expected agreement between two independent models trained on disjoint datasets unrelated to $D_u$. This is captured by the baseline agreement rate $\upsilon_{base}$, estimated as:
\begin{equation}\label{eq:verification metric baseline}
    \upsilon_{base}=\frac{\sum_{x\in D_u}\mathds{1}[M_{v_1}(x)=M_{v_2}(x)]}{|D_u|},
\end{equation}
where $M_{v_1}$ and $M_{v_2}$ are independent models trained on disjoint subsets of $D_r$. After obtaining $\upsilon_{base}$, $\zeta$ can be set by adding a safety margin to $\upsilon_{base}$ to account for variability in training and noise in $\upsilon$. Specifically, $\zeta$ is defined as:
\begin{equation}
    \zeta=\upsilon_{base}+k\cdot\sigma,
\end{equation}
where $\sigma = \sqrt{\frac{\upsilon_{base}(1 - \upsilon_{base})}{|D_u|}}$ is the approximate standard error of $\upsilon$, and $k$ is a confidence multiplier, e.g., $k = 2$ for $95\%$ confidence.
The actual standard error of $\upsilon$ is: $\sigma=\sqrt{\frac{p(1-p)}{|D_u|}}$, where $p$ is the true probability of agreement for a single sample $x \in D_u$ between $M_v(x)$ and $M_u(x)$, as $\upsilon$ follows the distribution: $\upsilon\sim\frac{{\rm Binomial}(|D_u|,p)}{|D_u|}$. However, since $p$ is unknown, we approximate it using $\upsilon_{base}$.


\vspace{1mm}
\noindent\textbf{Step 3. Verification Model Refinement.} When $\upsilon < \zeta$, the unlearned model $M_u$ cannot be trusted to have completely forgotten $D_u$. Given the inherent randomness in model training, 
the prover is allowed another opportunity to refine their verification model $M_v$. This refinement is guided by the accuracy difference between $M_v$ and $M_u$ on the unlearning data $D_u$.
Formally, let $Acc(M_v, D_u)$ and $Acc(M_u, D_u)$ be the accuracy of $M_v$ and $M_u$, respectively, on $D_u$. Comparing their accuracy can lead to three possible outcomes. 

First, if $Acc(M_v, D_u) < Acc(M_u, D_u)$, this suggests one of two possibilities: (1) the generalizability of $M_v$ is worse than $M_u$, or (2) $M_u$ retains information from $D_u$. The auditor then increases the size of the verification set $D_v$ by sampling additional data from $D_r$ to improve the generalizability of $M_v$. The exact number of additional data points can be determined based on the accuracy difference between $Acc(M_v, D_u)$ and $Acc(M_u, D_u)$, denoted as $\Delta_{acc} = Acc(M_u, D_u) - Acc(M_v, D_u)$. The size of the additional data to be added can be computed as: $\beta\Delta_{acc}|D_r|$, where $\beta \in (0,1)$ is a scaling coefficient. This formulation aims to enhance the generalizability of $M_v$ in a targeted manner, allowing $M_v$ to better approximate $M_u$ without over-correcting. Note that although the amount of additional data is computed linearly with respect to the accuracy gap $\Delta_{acc}$ and the size of $D_r$, this does not imply a linear relationship between model accuracy and dataset size. Rather, the computation is used only to determine a sufficient amount of data for aligning the behavior of $M_v$ with that of $M_u$. A formal analysis of this formulation is provided in Corollary \ref{cor:sufficiency}. After fine-tuning $M_v$ with the expanded $D_v$, the auditor re-evaluates its consistency with $M_u$. If $M_v$ and $M_u$ remain inconsistent, the auditor can conclude that $M_u$ contains information from $D_u$. 

Second, if $Acc(M_v, D_u) > Acc(M_u, D_u)$, this suggests that the unlearned model $M_u$ may be either over-unlearned or severely degraded, depending on the magnitude of the accuracy difference $\Delta_{acc} = Acc(M_v, D_u) - Acc(M_u, D_u)$. To address this, the verifier reduces the size of the verification set $D_v$ by removing a portion of its data and retrains $M_v$ on the reduced $D_v$. The number of data points to remove is determined as $\beta \Delta_{acc} |D_r|$. This adjustment intentionally reduces the performance of $M_v$ to align more closely with the over-unlearned or degraded behavior of $M_u$. After retraining $M_v$ on the reduced $D_v$, the auditor re-evaluates its consistency with $M_u$. If the two models remain inconsistent, the auditor concludes that $M_u$ retains information from $D_u$.

Third, if $Acc(M_v, D_u) = Acc(M_u, D_u)$, this means that the two models have comparable generalizability on $D_u$, yet they make different classification errors. This situation arises when $M_v$ and $M_u$ correctly classify similar proportions of $D_u$ but differ in the specific samples they classify correctly or incorrectly. Such differences indicate that while the overall performance metrics are aligned, the internal decision boundaries of the two models diverge. To address this, the auditor analyzes the classification results of $M_v$ and $M_u$ on $D_u$ to identify the subset of samples where their predictions differ: $D_{\text{diff}}=\{x\in D_u|M_v(x)\neq M_u(x)\}$. Next, the auditor samples additional data points from $D_r - D_v$, ensuring the exact amount is calculated as $(1 - \upsilon) |D_{\text{diff}}|$. This calculation is reasonable as $(1 - \upsilon)$ represents the proportion of disagreement between $M_v$ and $M_u$ on $D_u$, and scaling this by $|D_{\text{diff}}|$ ensures that the amount of additional data is proportional to the observed discrepancy. These newly sampled data are selected such that: $\{x\in D_r-D_v|M_v(x)=M_u(x)\}$. These data points are added to $D_v$, and then $M_v$ is fine-tuned on the updated $D_v$. After fine-tuning, the auditor re-evaluates the agreement between $M_v$ and $M_u$. If inconsistency persists, the auditor concludes that $M_u$ retains information from $D_u$

The auditing procedure is summarized in Algorithm \ref{alg:verification}. 

\begin{algorithm}\small
\caption{The entire auditing process}
\label{alg:verification}
\begin{algorithmic}[1]
\REQUIRE{Unlearned model $M_u$, Training data $D$, Unlearning data $D_u$;}
\ENSURE{The auditing result, whether $M_u$ contains information of $D_u$;}\vspace{1mm}

\textbf{Step 1. Verification model construction.}
\STATE The auditor collects a subset $D_v$ from $D-D_u$ and instructs the prover to train a model $M_v$ using $D_v$;\vspace{1mm}

\textbf{Step 2. Output comparison for Auditing.}
\STATE For each $x\in D_u$, the auditor checks whether $M_v(x)=M_u(x)$, and computes the verification metric $\upsilon$ using Eq. \ref{eq:verification metric};
\IF{$\upsilon\geq\zeta$}
    \STATE The auditor concludes that $M_u$ does not retain information from $D_u$;
\ELSE
    \STATE Proceed to Step 3;
\ENDIF\vspace{1mm}

\textbf{Step 3. Verification model refinement.}
\STATE The auditor calculates $Acc(M_v, D_u)$ and $Acc(M_u, D_u)$, then computes their difference as $\Delta = Acc(M_v, D_u) - Acc(M_u, D_u)$;
\IF{$\Delta<0$}
\STATE The auditor adds additional data to $D_v$, with the amount calculated as $\beta\Delta_{acc}|D_r|$, and instructs the prover to fine-tune $M_v$ using the expanded $D_v$;
\ELSIF{$\Delta>0$}
\STATE The auditor removes a portion of data from $D_v$, with the amount computed as $\beta\Delta_{acc}|D_r|$, and instructs the prover to retrain $M_v$ using the reduced $D_v$;
\ELSE
\STATE The auditor identifies the the subset of data as $D_{\text{diff}}=\{x\in D_u|M_v(x)\neq M_u(x)\}$;
\STATE The auditor samples data from $D-D_u-D_v$ with the quantity computed as $(1-\upsilon)|D_{\text{diff}}|$, ensuring that these newly sampled data satisfy: $\{x\in D-D_u-D_v|M_v(x)=M_u(x)\}$; 
\STATE The prover fine-tunes $M_v$ using the updated $D_v$;
\ENDIF
\STATE The auditor performs the auditing procedure as described in Step 2 by recalculating the verification metric $\upsilon$;
\IF{$\upsilon\geq\zeta$}
    \STATE The auditor concludes that $M_u$ does not retain information from $D_u$;
\ELSE
    \STATE The auditor concludes that $M_u$ retains information from $D_u$;
\ENDIF

\end{algorithmic}
\end{algorithm}

\section{Theoretical Analysis}\label{sec:analysis}

We first analyze the soundness of the preliminary check and then establish a theoretical connection between the auditing process and the concept of proof of ignorance. Specifically, the instance $q$ corresponds to the unlearned model $M_u$, and the solution $s$ represents its parameters. Likewise, the instance $q'$ is mapped to the verification model $M_v$, and the solution $s'$ to its parameters. Under this mapping, we demonstrate that (1) refining $M_v$ with additional data is a polynomial-time process; (2) if $M_v$ can be transformed into $M_u$ through such refinement, then $M_u$ can be regarded as having successfully unlearned the target data $D_u$; and (3) a single refinement iteration is sufficient to reach alignment under our framework. The formal proofs of these results are provided in the Appendix. 

\vspace{1mm}
\noindent\textbf{(0) Soundness of the Preliminary Check.} The check compares the behavior of the original model $M$ and the unlearned model $M_u$ on the unlearning data $D_u$. Formally, we have:

\begin{prop}[Sound rejection by preliminary check]\label{prop:preliminary check}
    Let $M$ be the original model trained on $D$ and $M_u$ be the claimed unlearned model. Suppose there exists a constant $\gamma > 0$ such that any successfully unlearned model $\widetilde{M}$ must satisfy
    \[
        A_u(M,\widetilde{M})\leq 1-\gamma,
    \]
    where
    \[
        A_u(M,\widetilde{M})=\frac{1}{|D_u|}\sum_{x\in D_u}\mathds{1}[M(x)=\widetilde{M}(x)].
    \]
    Then, for any $\epsilon \in (0,\gamma)$, if $A_u(M,M_u)\geq 1-\epsilon$, $M_u$ is not successfully unlearned.
\end{prop}
    
Proposition \ref{prop:preliminary check} provides a lightweight rejection mechanism for filtering out obvious failures of unlearning. Intuitively, if $M_u$ behaves almost identically to $M$ on the very data $D_u$ that are supposed to be forgotten, then the influence of $D_u$ must still persist in $M_u$. Therefore, such a model does not need to proceed to the full auditing protocol. Note that this proposition establishes only a one-sided guarantee: failing the preliminary check is sufficient to reject unlearning, but passing it does not imply successful unlearning. 

\vspace{1mm}
\noindent\textbf{(1) Polynomial Time Process.} Polynomial-time transformation is a core requirement in the proof-of-ignorance framework, ensuring that two solutions $s$ and $s'$ are computationally related such that ignorance of one implies ignorance of the other. In our setting, we require that the transformation between two models be achievable within polynomial time. Formally, we have the following theorem. 
\begin{thm}[Polynomial-time Refinement]\label{thm:Polynomial}
    A single refinement step that trains $M_v$ on an incremental dataset $D_\Delta$ for a fixed number of epochs executes in time $\mathrm{Poly}(|D_\Delta|, S)$, where $S$ denotes the number of model parameters.
\end{thm}

Theorem~\ref{thm:Polynomial} formally establishes that each refinement step, $M_v^{(t)} \mapsto M_v^{(t+1)}$, executes in polynomial time with respect to the size of the incremental data and the model complexity. Here, $M^{(t)}_v$ is the verification model after $t$ refinement steps. Consequently, any fixed number of refinement steps also operates within polynomial time, satisfying the requirement of the proof-of-ignorance analogy.

\vspace{2mm}
\noindent\textbf{(2) Auditing via Transformability of $M_v$ to $M_u$.} 

\begin{lem}[Output Difference Upper-bound]\label{lem:stability}
For empical risk minimization (ERM) or stochastic gradient descent (SGD) trained on two datasets whose symmetric difference has size $m$, yielding two models $M$ and $M'$, there exists a constant $C > 0$ such that
    \begin{equation}\nonumber
        \mathbb{E}_{x\sim\mathcal{X}}[||M(x)-M'(x)||_1]\leq C\cdot\frac{m}{|D_r|},
    \end{equation}
    where $M(x)$ is the predictive distribution of $M$ on input $x$ drawn from the data space $\mathcal{X}$.
\end{lem}

Lemma \ref{lem:stability} establishes an upper bound on the difference between the outputs of two models in terms of the difference in the sizes of their training sets. 

\begin{lem}[Reachability of a $D_r$-only solution]\label{lem:reachability}
Let $M^*_r$ be any ERM/SGD solution on $D_r$. For any $\epsilon > 0$, there exist incremental datasets $\{D_{\Delta_t}\}$ and $t = \mathrm{Poly}(1/\epsilon)$ such that
\begin{equation}\nonumber
    \mathbb{E}_{x\sim\mathcal{X}}[||M^{(t)}_v(x)-M^*_r(x)||_1]\leq\epsilon.
\end{equation}
\end{lem}

Note the conclusion of Lemma \ref{lem:reachability} also applies to solutions that depend only on subsets of $D_r$. The proof is analogous.

\begin{lem}[Alignment with any $D_r$-only (or subsets of $D_r$) unlearned models]\label{lem:alignment}
If $M_u$ is a $D_r$-only solution (or depends only on subsets of $D_r$), for any $\epsilon > 0$, there exists $t = \mathrm{Poly}(1/\epsilon)$ such that
\begin{equation}\nonumber
    \mathbb{E}_{x\sim\mathcal{X}}[||M^{(t)}_v(x)-M_u(x)||_1]\leq 2\epsilon.
\end{equation}
\end{lem}

We now provide the main conclusions about acceptance and rejection of $M_u$ as successful unlearning.
\begin{thm}[Acceptance via proof-of-ignorance reduction]\label{thm:accept}
    Given a tolerance $\delta > 0$, if there exists a polynomial-time refinement sequence that uses only $D_r$ (or any subsets of $D_r$) and produces $M^{(t)}_v$ such that $\mathbb{E}[||M^{(t)}_v - M_u||_1] \leq \delta$, then the unlearned model $M_u$ is accepted as having successfully unlearned $D_u$.
\end{thm}

\begin{thm}[Rejection under residual dependence]\label{thm:reject}
    If, for all polynomial-length refinement sequences based on $D_r$, we have $\mathbb{E}[||M^{(t)}_v - M_u||_1] > \delta$, then $M_u$ depends on $D_u$ and is rejected.
\end{thm}

Theorems \ref{thm:accept} and \ref{thm:reject} jointly establish the theoretical foundation to validate our approach.

\vspace{1mm}
\noindent\textbf{(3) Sufficiency of One Refinement Iteration.} Building on the above analysis, we arrive at the following conclusion.
\begin{cor}[One-step sufficiency]\label{cor:sufficiency}
    If $M_u$ is truly unlearned: it depends only on $D_r$ or subsets of $D_r$, then the first refinement using
    \begin{equation}\nonumber
        m_1\geq\frac{\lambda}{C}(1-\rho)\Delta_{acc}|D_r|
    \end{equation}
    samples is sufficient to reduce the gap between $M_u$ and $M^{(1)}_v$ below $\delta$, i.e., $\mathbb{E}[||M_u-M^{(1)}_v||_1]\leq\delta$, where $\lambda>0$ is a calibration constant. 
\end{cor}

Corollary \ref{cor:sufficiency} guarantees the one-step sufficiency for truly unlearned models $M_u$. For failed unlearning cases where $M_u$ retains knowledge from $D_u$, Theorem \ref{thm:reject} implies the persistent existence of a gap $\delta$ between $M_u$ and $M^{(t)}_v$. Hence, if $m_1 \geq \frac{\lambda}{C}(1-\rho)\Delta{acc}|D_r|$ still fails to bridge this gap, we conclude that $M_u$ has not been successfully unlearned.

\section{Experimental Setup}\label{sec:experiments}
\noindent\textbf{Datasets and Model Architectures.} We use three image datasets: \textbf{CIFAR10} \cite{Krizhevsky14}, \textbf{SVHN} \cite{Netzer11NIPS}, and \textbf{SkinCancer} \cite{SkinCancer}, and three text datasets: \textbf{BBCNews} \cite{BBCNews}, \textbf{AGNews} \cite{AGNews}, and \textbf{IMDB} \cite{IMDB}, covering both multi-class and binary classification tasks, to evaluate the multi-modal applicability of the proposed auditing framework. 

For image datasets, the model architecture consists of four CNN blocks followed by three fully connected layers. For text datasets, the model includes an embedding layer, three CNN blocks, and a fully connected layer. 

\vspace{1mm}
\noindent\textbf{Unlearning Methods.} We evaluate four representative unlearning methods, along with six of their variants, including \textbf{Retrain}, \textbf{Adversarial retrain \cite{Zhang24ICML}}, \textbf{Fine-tuning (FT) \cite{Koloskova25ICML}}, \textbf{Relabeling \cite{Graves21AAAI}}, \textbf{Relabeling with fine-tuning (Relabel+FT) \cite{Graves21AAAI}}, \textbf{Gradient ascent (GA) \cite{Thudi22EuroSP}}, \textbf{Gradient ascent with fine-tuning (GA+FT) \cite{Thudi22EuroSP}}, \textbf{Fisher forgetting \cite{Golatkar20CVPR}}, \textbf{Hessian forgetting \cite{Golatkar20CVPR,Golatkar20ECCV}}, and \textbf{Certified Hessian forgetting \cite{Zhang24ICMLb}}. Detailed descriptions of these unlearning approaches and datasets are provided in the Appendix.


\vspace{2mm}
\noindent\textbf{Evaluation Metrics.} Since our auditing method relies primarily on model outputs, the evaluation metrics are centered around comparing the outputs of different models.

\begin{itemize}[leftmargin=*]
\item \textbf{Model agreement.} This metric, defined in Eq. \ref{eq:verification metric}, quantifies the consistency between the outputs of two models. It is calculated as the ratio of the number of data points for which both models produce identical outputs to the total number of data points in the given dataset. This metric is widely used in the literature to evaluate model similarity. 

\item \textbf{Average Kullback–Leibler divergence (KL).} This metric complements the model agreement by measuring the average KL divergence between the output probability distributions of two models across a given dataset. 
\end{itemize}


\vspace{2mm}
\noindent\textbf{Baselines.} We adopt widely recognized auditing methods as baselines \cite{Xu23,Zhang24ICML}.

\begin{itemize}[leftmargin=*]
\item \textbf{Model Accuracy on the Unlearning Data.} This is a straightforward auditing method. The intuition is that if the unlearning data have been properly forgotten, the unlearned model should perform on them similarly to how it performs on the test data. 

\item \textbf{Membership Inference.} This method determines whether a given data sample was in the training set of a target model. If the model has truly forgotten the unlearning data, a membership inference attack should identify them as non-training data. We employ the state of the art technique proposed in \cite{Naderloui25USENIX}. 

\item \textbf{Backdoor-based Auditing.} This method leverages the principle behind backdoor attacks, where specific data are intentionally misclassified into targeted labels. When the unlearning data are backdoored, a successful unlearning process should eliminate their influence, meaning that the model should no longer misclassify these samples into the targeted labels. For a fair comparison, we adopt the same backdoor attack method used in \cite{Zhang24ICML}. 


\end{itemize}



\vspace{-1mm}
\section{Validation and Ablation Study}
\vspace{-0mm}
Before presenting the auditing results on existing unlearning methods, we first validate our auditing framework. The validation focuses on two aspects: \textbf{soundness} and \textbf{completeness}.
For soundness, we require that every model accepted by our framework as successfully unlearned indeed no longer retains any knowledge of the unlearning data.
For completeness, we require that every model that has truly unlearned the target data is not incorrectly rejected by our framework.

Empirically, soundness is evaluated through the false positive rate (\textbf{FPR}), while completeness is measured by the false negative rate (\textbf{FNR}). Specifically, to assess soundness, we train 100 partially unlearned models, each on a mixed dataset containing $40\%\sim 60\%$ of $D_r$ and $40\%\sim 60\%$ of $D_u$. To assess completeness, we train 100 truly unlearned models, each constructed using a randomly selected $60\%\sim 80\%$ subset of $D_r$. The exact sampling ratio is uniformly drawn from the interval $[a\%, b\%]$. The corresponding results are summarized in Table \ref{tab:validation}. The results show that the auditing framework achieves exceptionally low FPR and FNR values, both below $0.02$ across all six datasets, demonstrating its effectiveness.

\begin{table}[ht!]
\centering
\vspace{-0.5mm}
\caption{Validation of the proposed auditing framework in terms of soundness (FPR) and completeness (FNR).}
\label{tab:validation}
\begin{tabular}{lcccccc}
\toprule
 & CIFAR10 & SVHN & Skin & BBC & AG & IMDB \\
\midrule
FPR $\downarrow$ & $0.01$ & $0.02$ & $0$ & $0.01$ & $0$ & $0.01$ \\
FNR $\downarrow$ & $0.01$ & $0$ & $0.01$ & $0.02$ & $0.02$ & $0$ \\
\bottomrule
\vspace{-2mm}
\end{tabular}
\end{table}

We also evaluate the importance of both the preliminary checking step and the verification model refinement step. The preliminary step serves as an early diagnostic: if the unlearned model exhibits behavior nearly identical to the pre-unlearning model on the unlearning data, it can be immediately identified as a failure of unlearning. The validation results obtained after removing this preliminary step are reported in Table \ref{tab:validationNoCheck}. The results show that omitting the preliminary step leads to a significant increase in FPR across all six datasets. This is because some partially unlearned models can exhibit high agreement with the verification model, as their retained knowledge includes that of the verification model. These findings highlight the critical role of the preliminary step in filtering out unlearning failures at an early stage.

\begin{table}[ht!]
\centering
\vspace{-1mm}
\caption{Validation of the proposed auditing framework by removing the preliminary check.}
\label{tab:validationNoCheck}
\begin{tabular}{lcccccc}
\toprule
 & CIFAR10 & SVHN & Skin & BBC & AG & IMDB \\
\midrule
FPR $\downarrow$ & $0.13$ & $0.22$ & $0.18$ & $0.25$ & $0.23$ & $0.16$ \\
FNR $\downarrow$ & $0.01$ & $0$ & $0.02$ & $0.01$ & $0.02$ & $0.01$ \\
\bottomrule
\vspace{-5mm}
\end{tabular}
\end{table}

The refinement step (Step 3) provides the prover with an additional opportunity to fine-tune the verification model $M_v$, enabling it to more closely approximate the behavior of the unlearned model $M_u$. The validation results obtained after removing this refinement step are reported in Table \ref{tab:validationNoRefine}. The results show that omitting the refinement step leads to a notable increase in FNR across all six datasets. This is because, without refinement, the verification model remains under-optimized and fails to achieve high agreement with certain truly unlearned models. These findings demonstrate the essential role of the refinement step in ensuring the completeness of the auditing process.

\begin{table}[ht!]
\centering
\vspace{-1mm}
\caption{Validation of the proposed auditing framework by removing the refinement step.}
\label{tab:validationNoRefine}
\begin{tabular}{lcccccc}
\toprule
 & CIFAR10 & SVHN & Skin & BBC & AG & IMDB \\
\midrule
FPR $\downarrow$ & $0.01$ & $0.01$ & $0$ & $0.02$ & $0$ & $0.02$ \\
FNR $\downarrow$ & $0.11$ & $0.17$ & $0.24$ & $0.18$ & $0.20$ & $0.14$ \\
\bottomrule
\vspace{-2mm}
\end{tabular}
\end{table}

\vspace{-3mm}
\section{Experimental Results}

\subsection{Overall Results}
\vspace{0mm}
We now present the auditing results of our framework and the baselines across ten unlearning approaches with six datasets. 
The results are obtained by setting the size of the unlearning data $D_u$ to $20\%$ of the training set $D$, uniformly sampled from $D$. The initial size of the verification set $D_v$ is set to $10\%$ of the remaining data $D_r$. We also evaluated different initial sizes of $D_v$ and found that the verification outcomes remain consistent, since the auditor retains the flexibility to refine the verification model. In addition, the base agreement and base KL values are computed by setting the sizes of both $D_{v_1}$ and $D_{v_2}$ to $10\%$ of $D_r$. We also evaluated different ratios of $D_{v_1}$ and $D_{v_2}$ relative to $D_r$, ranging from $10\%$ to $50\%$. The results remain largely consistent across these settings. 
Finally, in the following tables, the number following the ``$@$'' symbol indicates the proportion of the remaining data $D_r$ used to obtain the results for each unlearning approach.

\begin{table}[!ht]\scriptsize
\vspace{-1mm}
	\centering
	\caption{Auditing results of our method on CIFAR10 across different unlearning approaches, where $D_u$, $D_r$, and $D_t$ denote the unlearning, remaining, and test data.}
\begin{tabular} {cccccc} 
\toprule
 \multirow{2}*{\makecell[c]{Unlearning\\Approaches}} & \multirow{2}*{\makecell[c]{Agree. on $D_u$\\(Base: $70.30$)}} & \multirow{2}*{\makecell[c]{KL on $D_u$\\(Base: $1.56$)}} & \multicolumn{3}{c}{Accuracy on} \\\cline{4-6}
 & & & $D_u$ & $D_r$ & $D_t$ \\
 \midrule
Pre-unlearn &  &  & $95.92$ & $95.61$ & $71.79$ \\\hline
Retrain$@0.35$ & $71.23$ & $1.22$ & $69.15$ & $99.54$ & $69.95$ \\ 
Adv. Retr.$@0.35$ & $71.10$ & $1.22$ & $73.87$ & $99.83$ & $71.61$ \\
Fine-tune$@0.3$ & $70.53$ & $1.42$ & $76.13$ & $100$ & $72.79$ \\
Relabel$@0.05$ & $30.90$ & $7.77$ & $46.88$ & $47.35$ & $40.97$ \\
Relab.+FT$@0.3$ & $70.94$ & $1.30$ & $73.67$ & $99.70$ & $71.40$ \\
GA$@0.05$ & $35.17$ & $6.20$ & $44.26$ & $45.08$ & $41.47$ \\
GA+FT$@0.3$ & $70.33$ & $1.36$ & $73.40$ & $99.86$ & $72.40$ \\
Fisher$@0.3$ & $68.96$ & $1.37$ & $95.87$ & $95.45$ & $71.62$ \\
Hessian$@0.35$ & $68.58$ & $1.37$ & $92.67$ & $92.64$ & $70.28$ \\
Cert.-Hess.$@0.3$ & $67.25$ & $1.91$ & $99.93$ & $99.93$ & $73.20$ \\
\bottomrule
\end{tabular}
	\label{tab:CIFAR10}
 \vspace{-1mm}
\end{table}

\begin{table}[!ht]\scriptsize
	\centering
	\caption{Auditing results of our method on SVHN across different unlearning approaches.}
\begin{tabular} {cccccc} 
\toprule
 \multirow{2}*{\makecell[c]{Unlearning\\Approaches}} & \multirow{2}*{\makecell[c]{Agree. on $D_u$\\(Base: $88.23$)}} & \multirow{2}*{\makecell[c]{KL on $D_u$\\(Base: $0.49$)}} & \multicolumn{3}{c}{Accuracy on} \\\cline{4-6}
 & & & $D_u$ & $D_r$ & $D_t$ \\
 \midrule
 Pre-unlearn &  &  & $99.49$ & $99.29$ & $89.94$ \\\hline
Retrain$@0.3$ & $88.81$ & $0.53$ & $89.06$ & $98.79$ & $87.29$ \\ 
Adv. Retr.$@0.3$ & $88.75$ & $0.44$ & $91.09$ & $99.63$ & $89.43$ \\
Fine-tune$@0.3$ & $89.14$ & $0.36$ & $91.68$ & $99.84$ & $90.50$ \\
Relabel$@0.05$ & $38.01$ & $9.61$ & $42.77$ & $42.75$ & $35.64$ \\
Relab.+FT$@0.35$ & $89.21$ & $0.44$ & $92.87$ & $99.75$ & $90.40$ \\
GA$@0.05$ & $42.57$ & $8.68$ & $44.26$ & $43.92$ & $40.83$ \\
GA+FT$@0.3$ & $89.35$ & $0.42$ & $90.26$ & $98.06$ & $89.45$ \\
Fisher$@0.35$ & $90.33$ & $0.35$ & $96.78$ & $96.44$ & $89.18$ \\
Hessian$@0.3$ & $89.35$ & $0.42$ & $99.24$ & $99.08$ & $89.98$ \\
Cert.-Hess.$@0.35$ & $88.45$ & $0.47$ & $99.75$ & $99.78$ & $90.72$ \\
\bottomrule
\end{tabular}
	\label{tab:SVHN}
 \vspace{-1mm}
\end{table}

\begin{table}[!ht]\scriptsize
	\centering
	\caption{Auditing results of our method on SkinCancer across different unlearning approaches.}
\begin{tabular} {cccccc} 
\toprule
 \multirow{2}*{\makecell[c]{Unlearning\\Approaches}} & \multirow{2}*{\makecell[c]{Agree. on $D_u$\\(Base: $91.62$)}} & \multirow{2}*{\makecell[c]{KL on $D_u$\\(Base: $0.44$)}} & \multicolumn{3}{c}{Accuracy on} \\\cline{4-6}
 & & & $D_u$ & $D_r$ & $D_t$ \\
 \midrule
 Pre-unlearn &  &  & $99.74$ & $99.77$ & $91.70$ \\\hline
Retrain$@0.3$ & $93.23$ & $0.30$ & $90.99$ & $99.92$ & $90.70$ \\ 
Adv. Retr.$@0.3$ & $92.71$ & $0.28$ & $91.10$ & $99.97$ & $90.80$ \\
Fine-tune$@0.3$ & $91.86$ & $0.36$ & $92.50$ & $99.91$ & $90.80$ \\
Relabel$@0.05$ & $46.02$ & $5.38$ & $55.08$ & $55.84$ & $47.40$ \\
Relab.+FT$@0.35$ & $91.65$ & $0.40$ & $90.63$ & $99.75$ & $90.60$ \\
GA$@0.05$ & $51.95$ & $5.11$ & $50.08$ & $50.61$ & $51.30$ \\
GA+FT$@0.3$ & $92.35$ & $0.32$ & $91.57$ & $99.90$ & $91.50$ \\
Fisher$@0.3$ & $90.42$ & $0.49$ & $99.64$ & $99.71$ & $91.60$ \\
Hessian$@0.3$ & $90.32$ & $0.49$ & $99.74$ & $99.70$ & $91.30$ \\
Cert.-Hess.$@0.3$ & $89.28$ & $0.55$ & $99.79$ & $99.97$ & $91.00$ \\
\bottomrule
\end{tabular}
	\label{tab:SkinCancer}
 \vspace{-2mm}
\end{table}

\vspace{1mm}
\noindent\textbf{Results of Our Auditing Framework.} The results are presented in Tables~\ref{tab:CIFAR10}, \ref{tab:SVHN}, \ref{tab:SkinCancer}, \ref{tab:BBCNews}, \ref{tab:AGNews}, and \ref{tab:IMDB}. According to our proposed auditing framework, unlearning approaches based on retraining or fine-tuning, such as Retrain, Adversarial Retrain, Fine-tune, Relabel with Fine-tune, and Gradient Ascent with Fine-tune, successfully achieve the unlearning objective. This is evidenced by the fact that the agreement and KL divergence between the corresponding unlearned models and the verification model are nearly identical to the base agreement and base KL divergence, which serve as reference benchmarks for ideal unlearning behavior. For example, in Table~\ref{tab:CIFAR10}, the five retraining/fine-tuning-based methods exhibit behavior distinct from the pre-unlearning model, especially on the unlearning dataset $D_u$. Specifically, their accuracies on $D_u$ are $69.15$ for Retrain, $73.87$ for Adversarial Retrain, $76.13$ for Fine-tune, $73.67$ for Relabel with Fine-tune, and $73.40$ for Gradient Ascent with Fine-tune, significantly lower than the pre-unlearning model's accuracy of $95.92$ on the same data.
Following the three-step main auditing protocol, the final agreement rates between the unlearned model $M_u$ and the verification model $M_v$ for these methods all exceed the base agreement threshold of $70.30$. 
Specifically, the agreement rates are $71.23$ for Retrain, $71.10$ for Adversarial Retrain, $70.53$ for Fine-tune, $70.94$ for Relabel with Fine-tune, and $70.33$ for Gradient Ascent with Fine-tune. Therefore, these methods are deemed to have achieved successful unlearning.

This auditing outcome can be further attributed to the nature of these methods. By either retraining the model from scratch or performing targeted fine-tuning on the remaining data, they effectively eliminate the influence of the unlearning data. Thus, the decision boundaries of the unlearned models align closely with those of the verification model.

We pay particular attention to the Fine-tune approach, which directly fine-tunes the model on the remaining data without explicitly interacting with or removing the unlearning data. At first glance, this method may appear inadequate from an algorithmic standpoint, as it continues training a model that initially retains knowledge of the unlearning data. However, our auditing results suggest that this approach can indeed achieve the unlearning objective.
The underlying reason lies in the phenomenon of catastrophic forgetting, where continued training on new data causes a model to lose previously acquired knowledge. By exclusively fine-tuning on the remaining data, this forgetting effect is concentrated on the unlearning data, erasing its influence from the model. 

Another approach that warrants close attention is Adversarial Retrain. Although this method is described as ``adversarial'' in \cite{Zhang24ICML}, our auditing results suggest that it effectively achieves the unlearning objective. The key mechanism is that it intentionally avoids using the unlearning data during retraining. Instead, it replaces each unlearning sample with a feature-similar sample selected from the remaining dataset.
We argue that this substitution strategy constitutes a valid unlearning mechanism. First, it strictly adheres to the unlearning requirement, as the actual unlearning data are excluded from the retraining process. Second, the use of feature-similar data does not violate the privacy or intent of data revocation. These substitute samples are drawn from the remaining, non-revoked data, and therefore do not reintroduce any personal information that the data owner has requested to remove. 

\begin{table}[!ht]\scriptsize
\vspace{-1mm}
	\centering
	\caption{Auditing results of our method on BBCNews across different unlearning approaches.}
\begin{tabular} {cccccc} 
\toprule
 \multirow{2}*{\makecell[c]{Unlearning\\Approaches}} & \multirow{2}*{\makecell[c]{Agree. on $D_u$\\(Base: $89.05$)}} & \multirow{2}*{\makecell[c]{KL on $D_u$\\(Base: $0.16$)}} & \multicolumn{3}{c}{Accuracy on} \\\cline{4-6}
 & & & $D_u$ & $D_r$ & $D_t$ \\
 \midrule
 Pre-unlearn &  &  & $99.40$ & $99.34$ & $90.03$ \\\hline
Retrain$@0.4$ & $92.62$ & $0.15$ & $89.90$ & $99.47$ & $89.88$ \\ 
Adv. Retr.$@0.4$ & $89.98$ & $0.25$ & $88.98$ & $99.46$ & $88.78$ \\
Fine-tune$@0.4$ & $89.92$ & $0.25$ & $90.15$ & $98.28$ & $88.80$ \\
Relabel$@0.05$ & $46.82$ & $0.71$ & $64.45$ & $75.46$ & $65.32$ \\
Relab.+FT$@0.4$ & $91.10$ & $0.22$ & $91.42$ & $99.47$ & $89.56$ \\
GA$@0.05$ & $51.92$ & $1.26$ & $64.62$ & $67.89$ & $62.12$ \\
GA+FT$@0.4$ & $91.37$ & $0.23$ & $91.20$ & $99.48$ & $89.06$ \\
Fisher$@0.4$ & $89.58$ & $0.33$ & $99.40$ & $99.36$ & $90.52$ \\
Hessian$@0.2$ & $73.85$ & $0.79$ & $82.40$ & $83.38$ & $77.92$ \\
Cert.-Hess.$@0.4$ & $88.88$ & $0.35$ & $99.40$ & $99.36$ & $89.96$ \\
\bottomrule
\end{tabular}
	\label{tab:BBCNews}
 \vspace{-1mm}
\end{table}

\begin{table}[!ht]\scriptsize
\vspace{-1mm}
	\centering
	\caption{Auditing results of our method on AGNews across different unlearning approaches.}
\begin{tabular} {cccccc} 
\toprule
 \multirow{2}*{\makecell[c]{Unlearning\\Approaches}} & \multirow{2}*{\makecell[c]{Agree. on $D_u$\\(Base: $88.16$)}} & \multirow{2}*{\makecell[c]{KL on $D_u$\\(Base: $0.47$)}} & \multicolumn{3}{c}{Accuracy on} \\\cline{4-6}
 & & & $D_u$ & $D_r$ & $D_t$ \\
 \midrule
Pre-unlearn &  &  & $99.89$ & $99.85$ & $89.90$ \\\hline
Retrain$@0.3$ & $90.88$ & $0.29$ & $89.44$ & $99.92$ & $89.6$ \\ 
Adv. Retr.$@0.3$ & $90.86$ & $0.28$ & $89.68$ & $99.95$ & $89.98$ \\
Fine-tune$@0.2$ & $89.17$ & $0.35$ & $91.06$ & $99.95$ & $89.85$ \\
Relabel$@0.05$ & $52.54$ & $2.34$ & $63.10$ & $67.09$ & $58.54$ \\
Relab.+FT$@0.3$ & $90.87$ & $0.30$ & $91.42$ & $99.95$ & $89.74$ \\
GA$@0.05$ & $60.11$ & $1.53$ & $67.53$ & $68.24$ & $64.28$ \\
GA+FT$@0.3$ & $90.79$ & $0.30$ & $91.10$ & $99.95$ & $89.55$ \\
Fisher$@0.35$ & $89.10$ & $0.46$ & $99.79$ & $99.78$ & $90.44$ \\
Hessian$@0.3$ & $82.99$ & $0.88$ & $88.45$ & $88.91$ & $82.70$ \\
Cert.-Hess.$@0.3$ & $88.06$ & $0.52$ & $99.88$ & $99.89$ & $90.16$ \\
\bottomrule
\end{tabular}
	\label{tab:AGNews}
 \vspace{-1mm}
\end{table}

\begin{table}[!ht]\scriptsize
	\centering
	\caption{Auditing results of our method on IMDB across different unlearning approaches.}
\begin{tabular} {cccccc} 
\toprule
 \multirow{2}*{\makecell[c]{Unlearning\\Approaches}} & \multirow{2}*{\makecell[c]{Agree. on $D_u$\\(Base: $79.03$)}} & \multirow{2}*{\makecell[c]{KL on $D_u$\\(Base: $0.73$)}} & \multicolumn{3}{c}{Accuracy on} \\\cline{4-6}
 & & & $D_u$ & $D_r$ & $D_t$ \\
 \midrule
Pre-unlearn &  &  & $99.95$ & $99.99$ & $79.90$ \\\hline
Retrain$@0.35$ & $83.60$ & $0.30$ & $79.90$ & $100$ & $80.34$ \\ 
Adv. Retr.$@0.35$ & $81.58$ & $0.40$ & $79.33$ & $99.98$ & $78.46$ \\
Fine-tune$@0.4$ & $79.70$ & $0.42$ & $81.10$ & $100$ & $77.52$ \\
Relabel$@0.05$ & $52.50$ & $0.16$ & $60.72$ & $72.95$ & $60.92$ \\
Relab.+FT$@0.4$ & $80.97$ & $0.42$ & $80.95$ & $99.96$ & $78.12$ \\
GA$@0.15$ & $64.72$ & $0.64$ & $72.15$ & $72.28$ & $65.04$ \\
GA+FT$@0.4$ & $80.00$ & $0.44$ & $80.92$ & $99.99$ & $78.10$ \\
Fisher$@0.4$ & $80.23$ & $0.55$ & $99.95$ & $99.98$ & $80.40$ \\
Hessian$@0.3$ & $71.05$ & $0.65$ & $77.90$ & $77.01$ & $71.26$ \\
Cert.-Hess.$@0.4$ & $79.80$ & $0.53$ & $99.95$ & $99.99$ & $80.40$ \\
\bottomrule
\end{tabular}
	\label{tab:IMDB}
 \vspace{-2mm}
\end{table}

In contrast, de-optimization-based approaches without subsequent fine-tuning, such as Relabel and Gradient Ascent, fail to achieve the true objective of machine unlearning and instead tend to degrade overall model utility. These methods attempt to ``unlearn'' by deliberately corrupting the model's performance on the unlearning data, typically through perturbing the decision boundaries or reversing the gradients associated with those samples. 
The fundamental flaw in these approaches is their reliance on performance deterioration as a proxy for forgetting. However, poor performance on the unlearning data does not guarantee that the model no longer retains knowledge of those data. In fact, a well-generalized model can perform well even on unseen data that shares the same distribution. 
Our auditing results support this argument. The models produced by these de-optimization-based approaches exhibit low agreement with verification models. For instance, in Table~\ref{tab:CIFAR10}, the agreement rates for Relabel and Gradient Ascent are only $30.90$ and $35.17$, respectively, substantially below the base agreement threshold of $70.30$. 

The Fisher, Hessian, and Certified Hessian Forgetting methods also fail to achieve the true objective of unlearning. Fisher Forgetting and Certified Hessian Forgetting can be rejected at an early stage by the preliminary check, as the resulting unlearned models exhibit behavior that is nearly indistinguishable from that of the pre-unlearning model.
For instance, in Table~\ref{tab:CIFAR10}, the accuracy of the Fisher-based unlearned model on the unlearning set $D_u$, the remaining set $D_r$, and the test set $D_t$ is $95.87$, $95.45$, and $71.62$, respectively, closely matching the pre-unlearning model's accuracy of $95.92$, $95.61$, and $71.79$.

In contrast, Hessian Forgetting can be readily filtered out by the preliminary check on the image datasets (Tables \ref{tab:CIFAR10}, \ref{tab:SVHN}, and \ref{tab:SkinCancer}), but it passes the preliminary check on the text datasets (Tables \ref{tab:BBCNews}, \ref{tab:AGNews}, and \ref{tab:IMDB}) as its accuracy on the unlearning data is substantially reduced. However, this reduction in accuracy is accompanied by a similarly significant drop in agreement with the verification model, and it is therefore still easily identified as a failed unlearning method by the subsequent auditing steps.
The underlying reason lies in the design of these approaches. They attempt to remove the influence of the unlearning data through local parameter perturbations derived from second-order approximations of the loss landscape, such as the Fisher information matrix or the Hessian matrix. However, such perturbations are often too limited to induce the behavioral shift required for true unlearning. In some cases, they merely preserve the original model behavior, causing the resulting model to remain highly similar to the pre-unlearning model. In other cases, they degrade model accuracy without genuinely removing the influence of the unlearning data.

We pay particular attention to the Certified Hessian unlearning method, as it fails to achieve effective unlearning despite being formally certified. Specifically, the certification refers to satisfying the $(\epsilon, \delta)$-unlearning criterion, which ensures that the unlearned model is, with high probability, close in parameter space to a model retrained from scratch. However, proximity in parameter space alone does not guarantee successful unlearning. While the two models may share similar characteristics, such as similar gradients or loss values, this does not imply that they encode identical knowledge, particularly on the unlearning data. 
More importantly, the $(\epsilon, \delta)$-certification does not explicitly verify whether the unlearned model has actually forgotten the targeted data. It merely quantifies how indistinguishable the unlearned model is from a reference model under a distributional assumption. Thus, an unlearned model may still retain detailed information about the unlearning data while satisfying the certification criteria. 

\begin{figure}[ht]
\vspace{-1mm}
\centering
    \includegraphics[scale=0.22]{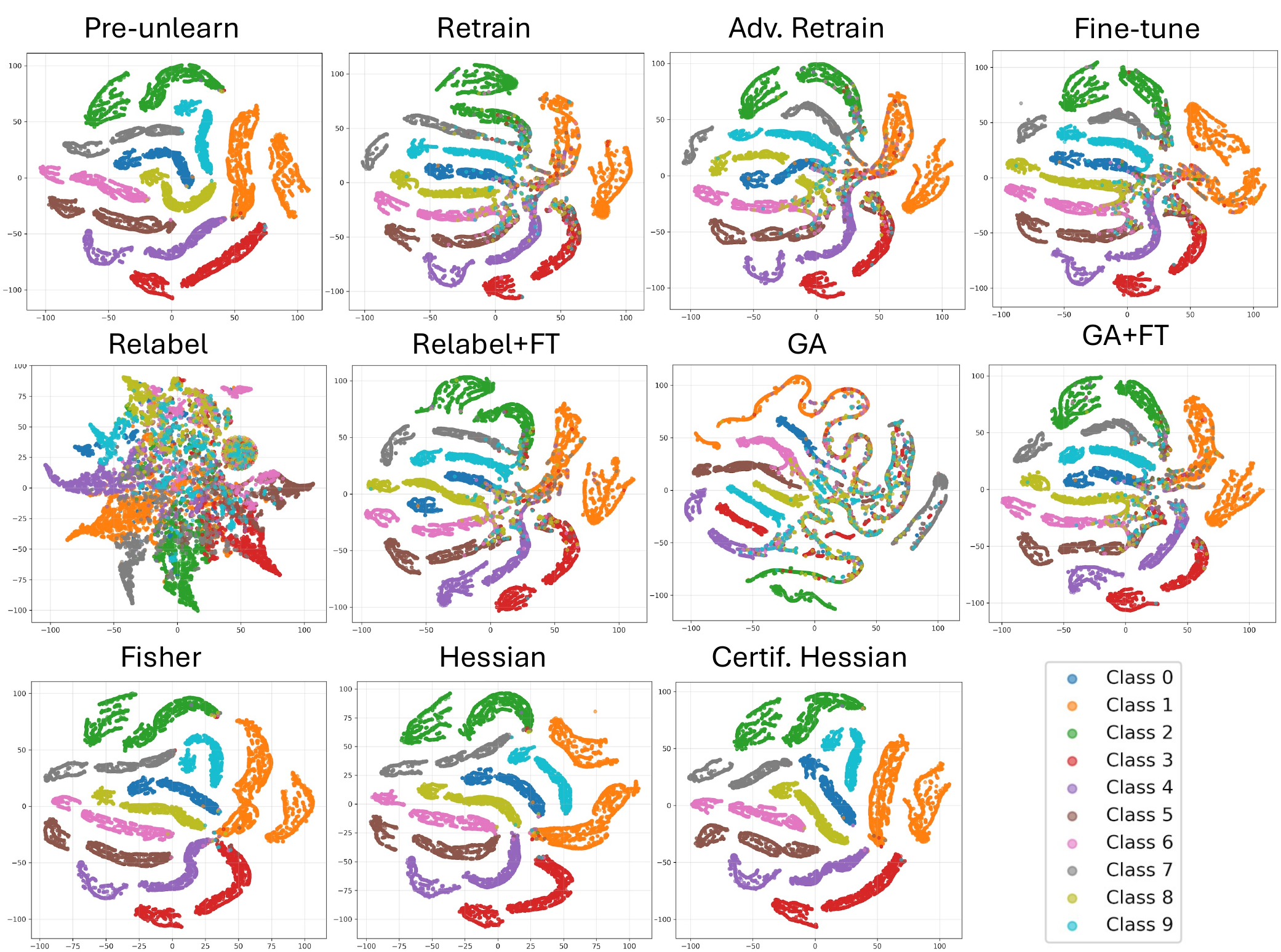}
	\caption{t-SNE visualization of classification output vectors on $D_u$ for various unlearned models on SVHN.}
   \vspace{-0mm}
	\label{fig:SVHN}
\end{figure}

\begin{figure}[ht]
\vspace{-1mm}
\centering
    \includegraphics[scale=0.22]{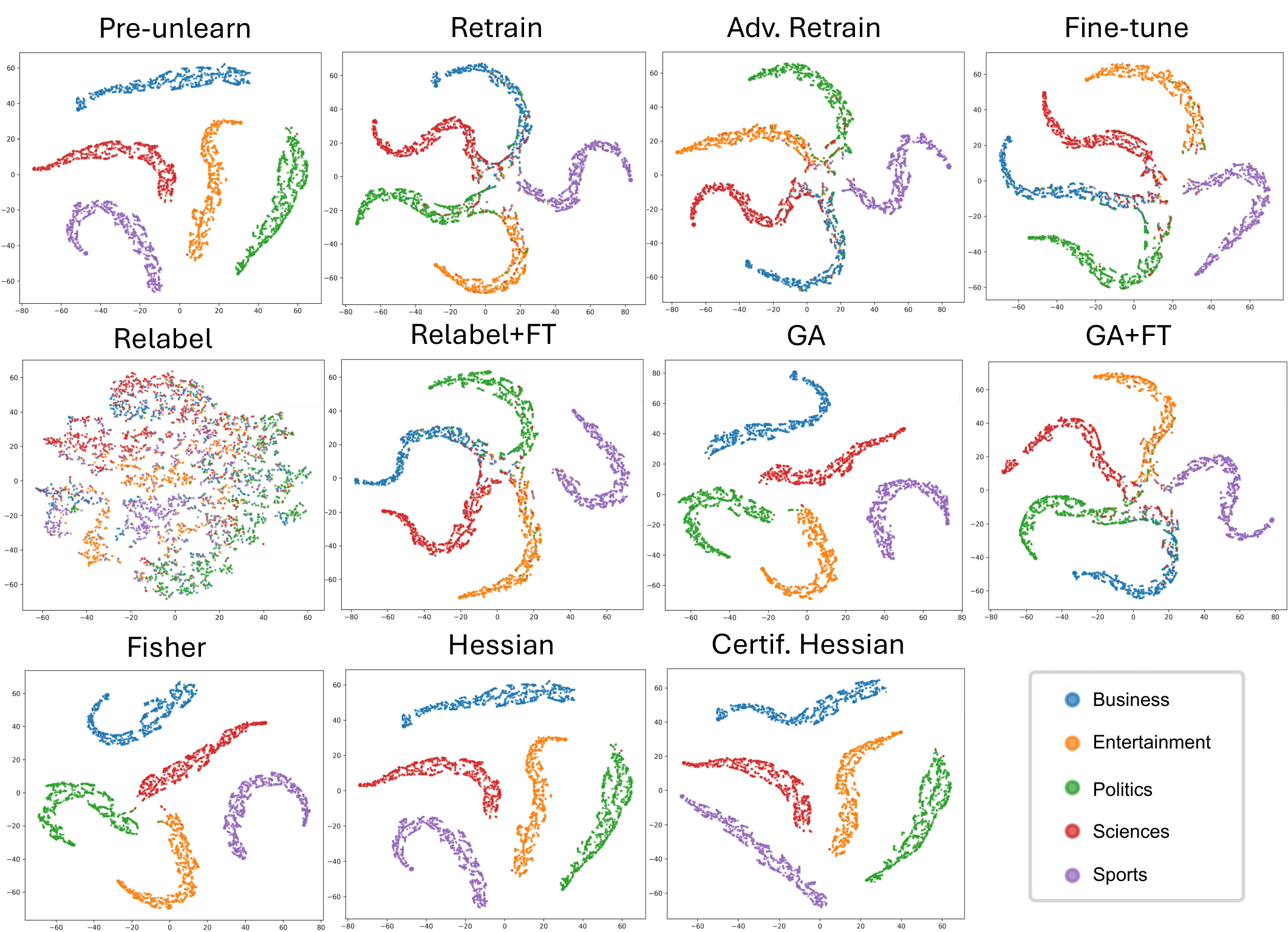}
	\caption{t-SNE visualization of classification output vectors on $D_u$ for various unlearned models on BBCNews.}
    \vspace{-0mm}
	\label{fig:BBCNews}
\end{figure}

Our auditing conclusions are further supported by the visualizations in Figures~\ref{fig:SVHN} and~\ref{fig:BBCNews}, where we employ t-SNE to project the output vectors of various unlearned models on $D_u$ for SVHN and BBCNews. As shown, models produced by Adversarial Retrain, Fine-tune, Relabel with Fine-tune, and Gradient Ascent with Fine-tune exhibit output distributions that closely resemble those of the retrained model, suggesting successful unlearning. In contrast, the Fisher, Hessian, and Certified Hessian methods yield visualizations that are remarkably similar to the pre-unlearning model, indicating that the influence of the unlearning data has likely not been removed. Moreover, models produced by Relabeling and Gradient Ascent display highly scattered patterns, suggesting degraded utility and unstable behavior. Although these visual observations do not constitute a formal audit of unlearning outcomes, they provide additional qualitative support for our quantitative auditing results. 

\vspace{1mm}
\noindent\textbf{Results of the Model Accuracy-based Auditing Approach.} The straightforward auditing method, Model Accuracy on the Unlearning Data, can produce results that are generally consistent with our findings. 
We observe from the tables that fine-tuning-based methods: Retrain, Adversarial Retrain, Fine-tuning, Relabeling with Fine-tuning, and Gradient Ascent with Fine-tuning consistently yield unlearned models whose accuracy on $D_u$ aligns more closely with that on $D_t$ than on $D_r$. This supports the conclusion that these methods successfully achieves the unlearning objective.
In contrast, for the remaining methods: Relabeling, Gradient Ascent, Fisher, Hessian, and Certified Hessian, the unlearned models tend to show accuracy on $D_u$ that is more similar to $D_r$ than to $D_t$. This indicates that these methods fail to unlearn $D_u$. 

Although this accuracy-based method can provide useful insights, it is not always reliable. In certain cases, it fails to offer a clear distinction. For example, in Table~\ref{tab:SkinCancer}, the Gradient Ascent method results in a model whose accuracy on the unlearning, remaining, and test data is nearly identical: $50.08$, $50.61$, and $51.30$, respectively. In such cases, the method cannot determine if successful unlearning has occurred. Similar ambiguity appears in other settings as well, such as with the Gradient Ascent method in Tables~\ref{tab:SkinCancerDu=0.3}. 

\begin{table*}[!ht]\scriptsize
\vspace{-1mm}
	\centering
	\caption{Backdoor-based auditing results across six datasets, where accuracy on $D_u$ denotes the proportion of samples correctly classified as backdoored labels.}
\begin{tabular} {cccc|ccc|ccc|ccc|ccc|ccc} 
\toprule
 \multirow{2}*{\makecell[c]{Unlearning\\Approaches}} & \multicolumn{3}{c|}{CIFAR10} & \multicolumn{3}{c|}{SVHN} & \multicolumn{3}{c|}{SkinCancer} & \multicolumn{3}{c|}{BBCNews} & \multicolumn{3}{c|}{AGNews} & \multicolumn{3}{c}{IMDB}\\\cline{2-19}
 & $D_u$ & $D_r$ & $D_t$ & $D_u$ & $D_r$ & $D_t$ & $D_u$ & $D_r$ & $D_t$ & $D_u$ & $D_r$ & $D_t$ & $D_u$ & $D_r$ & $D_t$ & $D_u$ & $D_r$ & $D_t$ \\
 \midrule
Pre-unlearn &  $100$ & $99.2$ & $73.2$ & $100$ & $97.5$ & $90.5$ & $99.0$ & $99.9$ & $91.3$ & $100$ & $99.4$ & $89.5$ & $100$ & $99.8$ & $89.8$ & $100$ & $99.9$ & $79.9$ \\\hline
Retrain & $6.7$ & $98.9$ & $71.1$ & $6.2$ & $97.0$ & $89.2$ & $56.8$ & $99.9$ & $91.8$ & $21.0$ & $99.4$ & $89.7$ & $24.9$ & $99.9$ & $89.9$ & $48.4$ & $100$ & $79.3$ \\ 
Adv. Retr. & $9.0$ & $99.2$ & $71.6$ & $6.4$ & $97.9$ & $90.5$ & $51.2$ & $99.8$ & $90.9$ & $20.7$ & $99.4$ & $88.8$ & $25.0$ & $99.9$ & $89.4$ & $52.5$ & $99.9$ & $78.2$ \\
Fine-tune & $10.9$ & $99.9$ & $72.9$ & $6.4$ & $99.8$ & $90.5$ & $51.3$ & $99.8$ & $91.1$ & $20.6$ & $99.4$ & $88.0$ & $25.8$ & $99.9$ & $89.6$ & $49.7$ & $100$ & $77.3$ \\
Relabel & $3.4$ & $54.5$ & $44.5$ & $12.6$ & $55.4$ & $51.9$ & $79.5$ & $79.1$ & $76.5$ & $15.1$ & $64.3$ & $58.6$ & $18.0$ & $69.0$ & $63.5$ & $100$ & $68.0$ & $62.2$ \\
Relab.+FT & $15.0$ & $99.9$ & $73.0$ & $6.9$ & $99.8$ & $90.4$ & $57.0$ & $100$ & $91.7$ & $20.6$ & $99.4$ & $88.1$ & $24.9$ & $99.9$ & $89.6$ & $45.8$ & $99.9$ & $76.7$ \\
GA & $13.0$ & $52.9$ & $46.4$ & $7.6$ & $56.7$ & $57.4$ & $10.0$ & $52.0$ & $50.0$ & $99.9$ & $77.7$ & $73.7$ & $99.9$ & $77.7$ & $73.7$ & $99.9$ & $65.9$ & $62.6$ \\
GA+FT & $13.1$ & $99.9$ & $72.8$ & $6.7$ & $99.8$ & $90.5$ & $61.9$ & $99.9$ & $91.5$ & $20.8$ & $99.4$ & $87.9$ & $24.1$ & $99.9$ & $89.9$ & $49.2$ & $100$ & $77.6$ \\
Fisher & $100$ & $99.1$ & $73.1$ & $100$ & $97.4$ & $90.5$ & $99.1$ & $99.9$ & $91.3$ & $100$ & $99.4$ & $90.3$ & $100$ & $99.8$ & $89.9$ & $100$ & $99.9$ & $80.8$ \\
Hessian & $100$ & $99.1$ & $73.1$ & $100$ & $97.4$ & $90.4$ & $98.9$ & $99.9$ & $91.3$ & $100$ & $62.5$ & $59.1$ & $99.7$ & $65.4$ & $61.3$ & $100$ & $60.0$ & $58.2$ \\
Cert.-Hess. & $100$ & $99.5$ & $73.8$ & $100$ & $98.0$ & $91.0$ & $99.0$ & $100$ & $91.0$ & $100$ & $99.4$ & $90.3$ & $100$ & $99.8$ & $90.0$ & $100$ & $99.9$ & $80.7$ \\
\bottomrule
\end{tabular}
	\label{tab:Backdoor}
 \vspace{-1mm}
\end{table*}

\vspace{1mm}
\noindent\textbf{Results of the Backdoor-based Auditing Approach.} The results are presented in Tables~\ref{tab:Backdoor}, where accuracy on $D_u$ denotes the proportion of samples correctly classified as backdoored labels. These results serve as strong empirical support for the conclusions drawn from our auditing method. In this setting, we apply a backdoor to all unlearning data, such that before unlearning, these data points are highly likely to be classified into their backdoored labels. After unlearning, a successful unlearning process should remove the influence of these poisoned associations, restoring the classification of unlearning data to their original, non-backdoored labels.

Consistent with this intuition, the results show that before unlearning, all the unlearning data are reliably classified into the backdoored labels across datasets. After unlearning, five methods: Retrain, Adversarial Retrain, Fine-tuning, Relabeling with Fine-tuning, and Gradient Ascent with Fine-tuning achieve extremely low accuracy on the backdoored labels for the unlearning data, while maintaining high accuracy on the remaining data and the test set. This strongly aligns with our conclusion: these fine-tuning-based methods are effective at forgetting the unlearning data. 

In contrast, Fisher, Hessian, and Certified Hessian Forgetting fail to erase the backdoor information. These methods still exhibit high accuracy in classifying the unlearning data to their backdoored labels after unlearning, indicating that the corresponding models retain knowledge of the unlearning data. 

For the Relabel and Gradient Ascent methods, although they reduce accuracy on backdoored labels after unlearning, they also significantly degrade performance on the remaining and test data. This suggests that the observed reduction in backdoor accuracy is not due to successful unlearning, but rather a side effect of harming overall model utility.


\begin{table*}[!ht]\scriptsize
\vspace{-1mm}
	\centering
	\caption{Membership inference-based auditing results across six datasets, where values denote the proportion of samples inferred as training members.}
\begin{tabular} {cccc|ccc|ccc|ccc|ccc|ccc} 
\toprule
 \multirow{2}*{\makecell[c]{Unlearning\\Approaches}} & \multicolumn{3}{c|}{CIFAR10} & \multicolumn{3}{c|}{SVHN} & \multicolumn{3}{c|}{SkinCancer} & \multicolumn{3}{c|}{BBCNews} & \multicolumn{3}{c|}{AGNews} & \multicolumn{3}{c}{IMDB}\\\cline{2-19}
 & $D_u$ & $D_r$ & $D_t$ & $D_u$ & $D_r$ & $D_t$ & $D_u$ & $D_r$ & $D_t$ & $D_u$ & $D_r$ & $D_t$ & $D_u$ & $D_r$ & $D_t$ & $D_u$ & $D_r$ & $D_t$ \\
 \midrule
Pre-unlearn & $88.6$ & $88.7$ & $67.9$ & $75.4$ & $75.5$ & $64.5$ & $99.8$ & $99.9$ & $92.4$ & $97.8$ & $99.9$ & $83.2$ & $99.1$ & $99.0$ & $84.8$ & $98.6$ & $98.8$ & $70.6$ \\\hline
Retrain & $65.4$ & $89.1$ & $65.1$ & $68.9$ & $75.0$ & $63.9$ & $92.0$ & $99.8$ & $92.1$ & $79.2$ & $98.3$ & $78.9$ & $81.0$ & $98.8$ & $81.3$ & $66.2$ & $98.9$ & $68.2$ \\ 
Adv. Retr. & $66.7$ & $88.8$ & $64.1$ & $69.3$ & $75.3$ & $64.2$ & $89.9$ & $99.7$ & $89.3$ & $77.0$ & $97.4$ & $76.7$ & $81.2$ & $98.4$ & $81.3$ & $62.9$ & $99.0$ & $63.6$ \\
Fine-tune & $66.2$ & $89.2$ & $62.9$ & $69.8$ & $76.5$ & $64.3$ & $89.8$ & $99.2$ & $88.1$ & $80.9$ & $96.7$ & $77.4$ & $83.2$ & $99.0$ & $82.3$ & $58.0$ & $97.6$ & $55.6$ \\
Relabel & $91.7$ & $92.2$ & $85.9$ & $86.1$ & $87.2$ & $81.2$ & $100$ & $99.9$ & $99.6$ & $96.5$ & $98.3$ & $94.7$ & $99.6$ & $99.7$ & $98.1$ & $92.3$ & $95.7$ & $82.0$ \\
Relab.+FT & $75.2$ & $90.8$ & $74.5$ & $71.6$ & $76.1$ & $67.5$ & $87.7$ & $98.0$ & $87.0$ & $77.2$ & $98.2$ & $74.6$ & $81.8$ & $98.8$ & $80.5$ & $62.9$ & $97.4$ & $61.6$ \\
GA & $84.5$ & $85.2$ & $83.5$ & $74.0$ & $74.5$ & $68.7$ & $99.7$ & $99.9$ & $99.0$ & $97.8$ & $99.0$ & $89.1$ & $99.1$ & $99.3$ & $92.1$ & $69.7$ & $70.2$ & $61.5$ \\
GA+FT & $48.3$ & $71.7$ & $46.8$ & $74.2$ & $74.2$ & $66.9$ & $86.8$ & $97.6$ & $87.5$ & $82.2$ & $98.5$ & $81.1$ & $82.9$ & $99.0$ & $81.4$ & $64.9$ & $99.0$ & $63.4$ \\
Fisher & $87.3$ & $87.1$ & $74.1$ & $74.5$ & $74.8$ & $66.7$ & $99.6$ & $99.6$ & $92.4$ & $98.2$ & $98.4$ & $86.1$ & $98.4$ & $98.5$ & $86.2$ & $98.2$ & $98.4$ & $69.7$ \\
Hessian & $86.9$ & $87.3$ & $75.1$ & $75.7$ & $76.2$ & $65.4$ & $99.5$ & $99.6$ & $93.0$ & $98.9$ & $99.0$ & $78.5$ & $99.6$ & $99.7$ & $98.8$ & $100$ & $100$ & $99.0$ \\
Cert.-Hess. & $88.9$ & $88.7$ & $65.3$ & $75.1$ & $75.7$ & $64.3$ & $99.7$ & $99.9$ & $93.5$ & $98.5$ & $98.6$ & $84.2$ & $98.9$ & $98.9$ & $84.8$ & $99.9$ & $98.3$ & $84.8$ \\
\bottomrule
\end{tabular}
	\label{tab:Membership}
 \vspace{-1mm}
\end{table*}

\vspace{1mm}
\noindent\textbf{Results of the Membership Inference-based Auditing Approach.} Membership inference-based method does not yield reliable results due to a high false positive rate, as shown in Tables~\ref{tab:Membership}. For example, after unlearning via the Retrain method, all unlearning data should ideally be inferred as non-members. Yet, in practice, most of them are still identified as members, indicating the method’s limited precision. 

However, comparing the membership inference results across different unlearning methods still provides useful insights. The proportion of unlearning data that are incorrectly inferred as members is significantly higher for the Relabel, Gradient Ascent, Fisher, Hessian, and Certified Hessian methods than for Retrain. This suggests that models unlearned using these methods retain more residual information about the unlearning data than retrained models. In other words, these methods fail to fully forget the unlearning data. 

\vspace{1mm}
\noindent\textbf{Summary.} Our auditing framework shows that retraining and fine-tuning-based approaches 
can successfully remove the influence of the unlearning data. In contrast, the Fisher, Hessian, and Certified Hessian methods consistently fail to achieve effective unlearning. The de-optimization-based approaches, Relabel and Gradient Ascent, merely degrade model performance without truly forgetting the targeted knowledge. 

\subsection{Robustness Study against Fake Unlearning}
\vspace{0mm}
To demonstrate the robustness of our auditing framework, we introduce an adversarial prover who attempts to bypass verification by replacing the true unlearning data with gradient-similar data. Specifically, the adversarial prover constructs a fake unlearning set $D_u'$ by selecting samples from the remaining data $D_r$ such that each sample in $D_u'$ exhibits gradient-level similarity to a corresponding sample in the original unlearning set $D_u$.

To construct $D_u'$, the adversarial prover first computes the gradient of the loss with respect to model parameters for each sample in $D_r$. Then, for each sample $x$ in $D_u$, the prover calculates its gradient and identifies the most similar sample $x'$ from $D_r$ based on cosine similarity between their gradients. The selected $x'$ is then added to $D_u'$.
To ensure uniqueness, if the most similar candidate has already been used, the prover selects the next most similar unused sample. This process is repeated iteratively until all samples in $D_u$ are matched with distinct and gradient-similar counterparts from $D_r$. 

During unlearning, the adversarial prover unlearns the fake unlearning set $D_u'$ instead of the true unlearning set $D_u$, aiming to retain the knowledge of $D_u$ while evading our auditing. To give the adversary the best chance of success, we allow the use of the retraining-from-scratch approach, which has been shown to be the most likely to pass the verification based on the above results. We refer to this method as \textbf{Fake Retraining}. 

Another fake unlearning approach is from \cite{Zhang24ICML}. This method forges a plausible ``proof of retraining'' without actually performing full retraining on the remaining data $D_r$. If the current training data do not contain any unlearning samples, the model is updated using a single sample randomly selected from $D_r$. However, if the training data include unlearning samples, they are replaced with their nearest neighbors from $D_r$, and the model is updated using these substitutes. 
We refer to this method as \textbf{Forging-based Retraining}.
The results are presented in Tables~\ref{tab:CIFAR10Attack}, \ref{tab:SVHNAttack}, \ref{tab:SkinCancerAttack}, \ref{tab:BBCNewsAttack}, \ref{tab:AGNewsAttack}, and \ref{tab:IMDBAttack}.

\begin{table}[!ht]\scriptsize
\vspace{-0mm}
	\centering
	\caption{Auditing results of our method on CIFAR10 under the two fake unlearning approaches.}
\begin{tabular} {cccccc} 
\toprule
 \multirow{2}*{\makecell[c]{Unlearning\\Approaches}} & \multirow{2}*{\makecell[c]{Agree. on $D_u$\\(Base: $70.30$)}} & \multirow{2}*{\makecell[c]{KL on $D_u$\\(Base: $1.56$)}} & \multicolumn{3}{c}{Accuracy on} \\\cline{4-6}
 & & & $D_u$ & $D_r$ & $D_t$ \\
 \midrule
Pre-unlearn &  &  & $95.92$ & $95.61$ & $71.79$ \\\hline
Retrain$@0.4$ & $71.23$ & $1.22$ & $69.15$ & $99.54$ & $69.95$ \\ 
Fake Retr.$@0.4$ & $67.33$ & $1.89$ & $99.67$ & $93.03$ & $69.85$ \\
Forge$@0.3$ & $65.67$ & $2.29$ & $75.20$ & $75.25$ & $67.89$ \\
\bottomrule
\end{tabular}
	\label{tab:CIFAR10Attack}
 \vspace{-0mm}
\end{table}

\begin{table}[!ht]\scriptsize
	\centering
	\caption{Auditing results of our method on SVHN under the two fake unlearning approaches.}
\begin{tabular} {cccccc} 
\toprule
 \multirow{2}*{\makecell[c]{Unlearning\\Approaches}} & \multirow{2}*{\makecell[c]{Agree. on $D_u$\\(Base: $88.23$)}} & \multirow{2}*{\makecell[c]{KL on $D_u$\\(Base: $0.49$)}} & \multicolumn{3}{c}{Accuracy on} \\\cline{4-6}
 & & & $D_u$ & $D_r$ & $D_t$ \\
 \midrule
 Pre-unlearn &  &  & $99.49$ & $99.29$ & $89.94$ \\\hline
Retrain$@0.35$ & $89.01$ & $0.53$ & $89.06$ & $98.79$ & $87.29$ \\ 
Fake Retr.$@0.35$ & $88.48$ & $0.46$ & $99.52$ & $97.75$ & $89.75$ \\
Forge$@0.3$ & $82.72$ & $0.51$ & $92.91$ & $92.60$ & $88.14$ \\
\bottomrule
\end{tabular}
	\label{tab:SVHNAttack}
 \vspace{-1mm}
\end{table}

\begin{table}[!ht]\scriptsize
	\centering
	\caption{Auditing results of our method on SkinCancer under the two fake unlearning approaches.}
\begin{tabular} {cccccc} 
\toprule
 \multirow{2}*{\makecell[c]{Unlearning\\Approaches}} & \multirow{2}*{\makecell[c]{Agree. on $D_u$\\(Base: $91.62$)}} & \multirow{2}*{\makecell[c]{KL on $D_u$\\(Base: $0.44$)}} & \multicolumn{3}{c}{Accuracy on} \\\cline{4-6}
 & & & $D_u$ & $D_r$ & $D_t$ \\
 \midrule
 Pre-unlearn &  &  & $99.74$ & $99.77$ & $91.70$ \\\hline
Retrain$@0.4$ & $93.23$ & $0.30$ & $90.99$ & $99.92$ & $90.70$ \\ 
Fake Retr.$@0.4$ & $90.42$ & $0.40$ & $99.69$ & $97.64$ & $90.60$ \\
Forge$@0.3$ & $84.57$ & $0.48$ & $93.13$ & $94.07$ & $91.40$ \\
\bottomrule
\end{tabular}
	\label{tab:SkinCancerAttack}
 \vspace{-1mm}
\end{table}

We observe that both fake unlearning approaches can be readily identified as failed unlearning. Specifically, the Fake Retraining approach can be directly detected by the preliminary check, as the unlearned model exhibits accuracies on the unlearning, remaining, and test sets that are consistently highly aligned with those of the pre-unlearning model. 

In contrast, the Forging-based Retraining approach attains model accuracies on the text datasets (Tables \ref{tab:BBCNewsAttack}, \ref{tab:AGNewsAttack}, and \ref{tab:IMDBAttack}) that are similarly close to those of the pre-unlearning model, and is therefore also directly detected by the preliminary check. On the image datasets (Tables \ref{tab:CIFAR10Attack}, \ref{tab:SVHNAttack}, and \ref{tab:SkinCancerAttack}), however, this approach yields lower accuracies than the pre-unlearning model. Nevertheless, this reduction in accuracy is accompanied by a corresponding decrease in agreement with the verification model, falling below the base agreement threshold and thus still indicating a failure of unlearning.

\begin{table}[!ht]\scriptsize
\vspace{-0mm}
	\centering
	\caption{Auditing results of our method on BBCNews under the two fake unlearning approaches.}
\begin{tabular} {cccccc} 
\toprule
 \multirow{2}*{\makecell[c]{Unlearning\\Approaches}} & \multirow{2}*{\makecell[c]{Agree. on $D_u$\\(Base: $89.05$)}} & \multirow{2}*{\makecell[c]{KL on $D_u$\\(Base: $0.16$)}} & \multicolumn{3}{c}{Accuracy on} \\\cline{4-6}
 & & & $D_u$ & $D_r$ & $D_t$ \\
 \midrule
 Pre-unlearn &  &  & $99.40$ & $99.34$ & $90.03$ \\\hline
Retrain$@0.4$ & $92.62$ & $0.15$ & $89.90$ & $99.47$ & $89.88$ \\ 
Fake Retr.$@0.4$ & $88.75$ & $0.36$ & $99.62$ & $97.19$ & $90.56$ \\
Forge$@0.4$ & $87.50$ & $0.34$ & $99.48$ & $99.34$ & $90.58$ \\
\bottomrule
\end{tabular}
	\label{tab:BBCNewsAttack}
 \vspace{-1mm}
\end{table}

\begin{table}[!ht]\scriptsize
	\centering
	\caption{Auditing results of our method on AGNews under the two fake unlearning approaches.}
\begin{tabular} {cccccc} 
\toprule
 \multirow{2}*{\makecell[c]{Unlearning\\Approaches}} & \multirow{2}*{\makecell[c]{Agree. on $D_u$\\(Base: $88.16$)}} & \multirow{2}*{\makecell[c]{KL on $D_u$\\(Base: $0.47$)}} & \multicolumn{3}{c}{Accuracy on} \\\cline{4-6}
 & & & $D_u$ & $D_r$ & $D_t$ \\
 \midrule
Pre-unlearn &  &  & $99.89$ & $99.85$ & $89.90$ \\\hline
Retrain$@0.45$ & $90.88$ & $0.29$ & $89.44$ & $99.92$ & $89.60$ \\ 
Fake Retr.$@0.4$ & $88.99$ & $0.48$ & $99.92$ & $98.06$ & $90.03$ \\
Forge$@0.3$ & $87.58$ & $0.52$ & $99.77$ & $99.75$ & $90.38$ \\
\bottomrule
\end{tabular}
	\label{tab:AGNewsAttack}
 \vspace{-1mm}
\end{table}

\begin{table}[!ht]\scriptsize
	\centering
	\caption{Auditing results of our method on IMDB under the two fake unlearning approaches.}
\begin{tabular} {cccccc} 
\toprule
 \multirow{2}*{\makecell[c]{Unlearning\\Approaches}} & \multirow{2}*{\makecell[c]{Agree. on $D_u$\\(Base: $79.03$)}} & \multirow{2}*{\makecell[c]{KL on $D_u$\\(Base: $0.73$)}} & \multicolumn{3}{c}{Accuracy on} \\\cline{4-6}
 & & & $D_u$ & $D_r$ & $D_t$ \\
 \midrule
Pre-unlearn &  &  & $99.95$ & $99.99$ & $79.90$ \\\hline
Retrain$@0.4$ & $83.60$ & $0.30$ & $79.90$ & $100$ & $80.34$ \\ 
Fake Retr.$@0.35$ & $78.53$ & $0.52$ & $99.98$ & $94.51$ & $78.68$ \\
Forge$@0.3$ & $76.08$ & $0.51$ & $100$ & $100$ & $82.36$ \\
\bottomrule
\end{tabular}
	\label{tab:IMDBAttack}
 \vspace{-1mm}
\end{table}

\begin{figure}[ht!]
\vspace{-0mm}
\centering
	\begin{minipage}[c]{1\textwidth}
    \subfigure[\small{SVHN}]{
    \centering\includegraphics[scale=0.21]{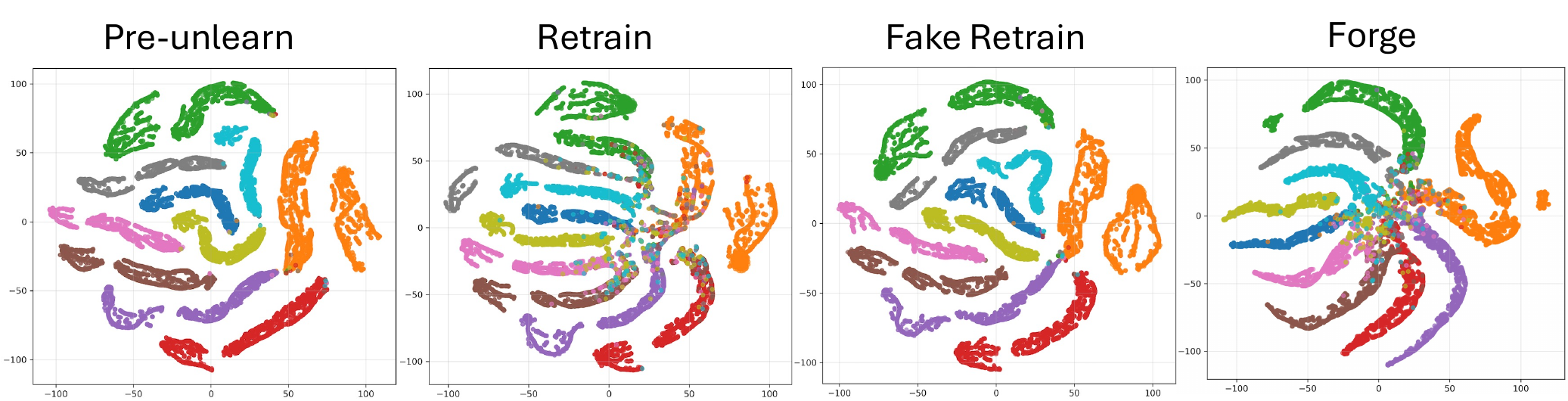}
			\label{fig:SVHNAttack}}\\[-3mm]
	\subfigure[\small{BBCNews}]{
    \centering\includegraphics[scale=0.205]{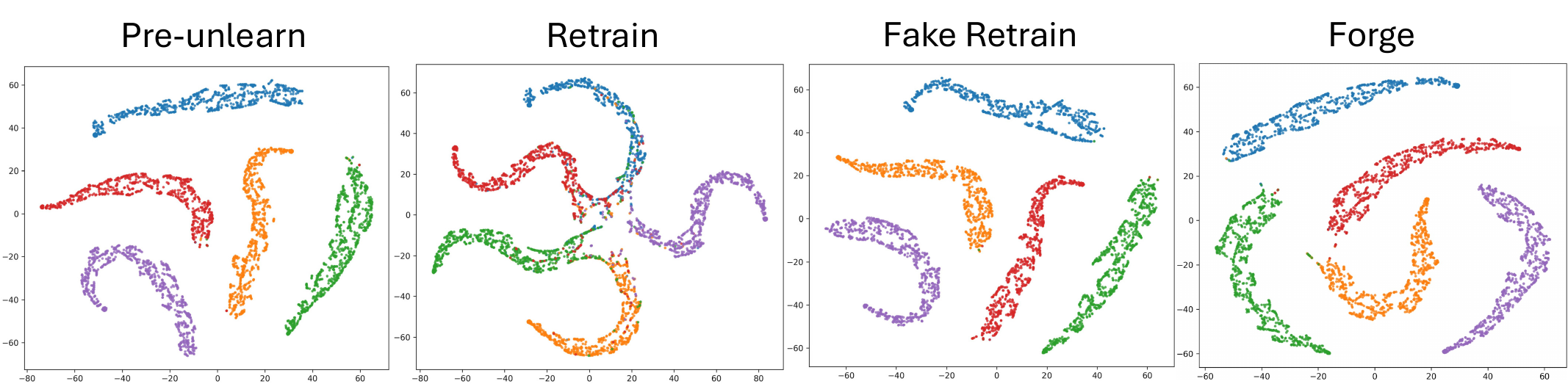}
			\label{fig:BBCNewsAttack}}
    \end{minipage}
	\caption{t-SNE visualization of classification output vectors on $D_u$ across SVHN and BBCNews.}
	\vspace{-2mm}
	\label{fig:Attack}
\end{figure}

The auditing conclusions can also be visually observed in Figure \ref{fig:Attack}. The t-SNE visualization shows that the output distribution of the Fake Retrained model closely resembles that of the pre-unlearning model. In contrast, the output distribution of the model produced by the Forging-based approach exhibits noticeable randomness, particularly in SVHN likely due to overall performance degradation.

\vspace{1mm}
\noindent\textbf{Summary.} Our auditing framework can easily identify fake unlearning approaches, especially those that superficially substitute data or forge proofs of retraining.

\vspace{-1mm}
\subsection{Other Study}
We also conducted adaptability studies, which evaluate the auditing framework across varying sizes and distributions of the unlearning data $D_u$. The detailed results are provided in the Appendix. 

\vspace{-1mm}
\section{Extension to Large Language Models}\label{sec:LLM}
\vspace{-0mm}
The proposed auditing method can also be extended to LLMs. Although a current trend in auditing LLM unlearning is to use adversarial auditing methods, such as membership inference, existing studies have shown that these approaches suffer from limited reliability \cite{Chowdhury25SaTML,Zhang25SaTML,Naderloui25USENIX}. Thus, our auditing framework can fill this gap.  

We consider a scenario where a pre-trained LLM, $M$, is fine-tuned on a dataset $D$, resulting in a fine-tuned model $M'$. Now, suppose that a subset $D_u$ of $D$ is requested to be unlearned, and the model provider applies an unlearning algorithm to remove $D_u$, producing an unlearned model $M'_u$.
To audit whether the unlearning was successful, the model provider, acting as the prover, uses the original pre-trained LLM, $M$, as the verification model. Since $M$ was never trained on $D_u$, it serves as a guaranteed verification model that excludes any information from $D_u$. The auditor then compares the outputs of $M$ and $M'_u$ using the samples in $D_u$ as input. If the outputs are consistent, the auditor concludes that $M'_u$ does not retain information from $D_u$. Otherwise, the model provider is given another opportunity to refine the verification model. 

We follow the experimental setup in \cite{Yao24NeurIPS} to evaluate their proposed unlearning method. Their approach integrates key ideas from three unlearning methods: Gradient Ascent, Relabel (referred to as Random Mismatch in their work), and Fine-tuning (referred to as Maintain Performance). 

For the LLM, we adopt OPT-2.7B \cite{OPT2.7B}. 
Regarding the datasets, we use PKU-SafeRLHF \cite{Ji23NeurIPS} as the unlearning dataset $D_u$ and TruthfulQA \cite{Lin22ACL} as the remaining dataset $D_r$, both of which consist of question-answer (Q\&A) pairs. The combined dataset $D = D_u \cup D_r$ is used to fine-tune the pre-trained OPT-2.7B model. For auditing, we adopt BERTScore \cite{Zhang20ICLR} and KL divergence to quantify the output differences between the unlearned model and the verification model. The results are presented in Table ~\ref{tab:OPT2.7B}. 

\begin{table}[!ht]\scriptsize
\vspace{-1mm}
	\centering
	\caption{Auditing results of our method on LLMs.}
\begin{tabular} {ccccc} 
\toprule
 \multirow{2}*{\makecell[c]{Unlearning\\Approaches}} & \multicolumn{2}{c}{$D_u$} & \multicolumn{2}{c}{$D_r$} \\\cline{2-5}
 & BERT & KL & BERT & KL \\
 \midrule
Base$@0.1$ & $0.93$ & $4.2\times 10^{-5}$ & $0.93$ & $5.9\times 10^{-5}$  \\\hline
Retrain$@0.4$ & $0.86$ & $3.0\times 10^{-4}$ & $0.84$ & $3.6\times 10^{-4}$  \\ 
GA$@0.4$ & $0.72$ & $8.1\times 10^{-2}$ & $0.66$ & $6.4\times 10^{-2}$ \\
GA+FT$@0.4$ & $0.87$ & $6.3\times 10^{-3}$ & $0.90$ & $4.2\times 10^{-3}$ \\
GA+Relab.+FT$@0.4$ & $0.52$ & $1.9\times 10^{-1}$ & $0.49$ & $1.8\times 10^{-1}$ \\
\bottomrule
\end{tabular}
	\label{tab:OPT2.7B}
 \vspace{-1mm}
\end{table}

The auditing results on LLMs are highly consistent with those observed in conventional models. Both the Retrain and Gradient Ascent with Fine-tuning approaches yield BERTScores that are closely aligned with the base BERTScore, indicating successful unlearning. In contrast, the Gradient Ascent approach alone results in significantly lower BERTScores compared to the base, highlighting large discrepancies in the outputs and signaling a failure of unlearning.

A particularly interesting case arises with the method that combines Gradient Ascent, Relabel, and Fine-tuning. According to our auditing framework, this combined approach does not constitute a successful unlearning strategy. This is because the different components of the method target conflicting objectives. Gradient Ascent pushes the model to forget certain behaviors by increasing the loss on the unlearning data, while Relabel introduces mismatched supervision signals, and Fine-tuning attempts to recover overall performance. When combined, these operations may interfere with one another, failing to forget the unlearning data. 


\section{Related Work}
\vspace{-0mm}
\noindent\textbf{Machine Unlearning.} Machine unlearning include two types: exact unlearning and approximate unlearning \cite{Thudi23}.
Exact unlearning typically involves retraining the model from scratch with the unlearned data removed from the training set. While this process can be optimized by partitioning the training set into multiple shards and retraining only on the shard containing the unlearned data \cite{Bourtoule21, Chen22CCS}, it remains computationally expensive and impractical for large-scale models.
In contrast, approximate unlearning seeks to modify the model's parameters by applying updates over a limited number of epochs. The goal is to ensure that the modified model closely approximates the one obtained through full retraining in the parameter space \cite{Golatkar21CVPR, Graves21AAAI, Li25WWW}. 
Research on machine unlearning has recently expanded into the LLM domain \cite{Liu24NatureMI,Wang25ICLR}. While approximate unlearning remains an efficient approach \cite{Chen23EMNLP,Yao24NeurIPS}, a new form has emerged, focusing on controlling LLM outputs through prompt engineering \cite{Pawelczyk24ICML,Liu24NIPS}. 


\vspace{1mm}
\noindent\textbf{Auditing Machine Unlearning.} Machine unlearning can be intuitively audited by comparing the unlearned model to a reference model retrained from scratch either in parameter space or behavior space \cite{Thudi23,Xue25}. However, this approach assumes access to a freshly retrained model, which is often impractical due to the high computational cost. 
Alternative lines of work include certified unlearning \cite{Guo20, Sekhari21NIPS}, which leverages differential privacy \cite{Dwork14} to obscure the differences between the retrained and unlearned models, and reproducing verification \cite{Weng24TIFS, Zhang24ICML, Eisenhofer25}, which records the unlearning operations, enabling the data owner to reproduce each step and verify the integrity of the process. However, certified unlearning relies on the convexity of the loss function, limiting its applicability to more complex models. Although recent works \cite{Zhang24ICMLb, Koloskova25ICML} have extended certified unlearning to deep learning settings, these guarantees hold only when using their specific unlearning algorithms, and thus do not constitute a universal certification framework. Meanwhile, reproducing verification focuses on auditing the unlearning procedure rather than the unlearning outcomes.


In the ML security domain, one intuitive approach to auditing unlearning is through membership inference attacks \cite{Shokri17}. However, membership inference typically requires training a large number of shadow models, resulting in substantial computational overhead. Moreover, its inference accuracy remains insufficient for reliable verification in practical settings \cite{Carlini22SP,Naderloui25USENIX}. 
Another auditing strategy is backdoor-based verification, where data owners deliberately inject poisoned data into the training set to serve as backdoor triggers \cite{Sommer22PETS}. 
However, this method requires the injection of a sufficient amount of poisoned data into the training set prior to model training. 

\vspace{-1mm}
\section{Conclusion}\label{sec:conclusion}
\vspace{-0mm}
This paper has addressed a fundamental yet underexplored problem in machine unlearning: how to audit if unlearning has been successfully achieved. We propose the first universal auditing framework, inspired by the concept of proof of ignorance, that can effectively assess a wide range of unlearning methods across diverse datasets. Our framework applies not only to conventional ML models but also extends to LLMs. Future research will focus on extending this framework to support streaming unlearning under continual learning settings. 







\vspace{-0mm}
\bibliographystyle{IEEEtran}
\bibliography{references}

\begin{thebibliography}{10}
\providecommand{\url}[1]{#1}
\csname url@samestyle\endcsname
\providecommand{\newblock}{\relax}
\providecommand{\bibinfo}[2]{#2}
\providecommand{\BIBentrySTDinterwordspacing}{\spaceskip=0pt\relax}
\providecommand{\BIBentryALTinterwordstretchfactor}{4}
\providecommand{\BIBentryALTinterwordspacing}{\spaceskip=\fontdimen2\font plus
\BIBentryALTinterwordstretchfactor\fontdimen3\font minus \fontdimen4\font\relax}
\providecommand{\BIBforeignlanguage}[2]{{%
\expandafter\ifx\csname l@#1\endcsname\relax
\typeout{** WARNING: IEEEtran.bst: No hyphenation pattern has been}%
\typeout{** loaded for the language `#1'. Using the pattern for}%
\typeout{** the default language instead.}%
\else
\language=\csname l@#1\endcsname
\fi
#2}}
\providecommand{\BIBdecl}{\relax}
\BIBdecl

\bibitem{GDPR}
GDPR, ``{General Data Protection Regulation},'' \emph{https://gdpr-info.eu}, 2016.

\bibitem{Cao15}
Y.~Cao and J.~Yang, ``{Towards Making Systems Forget with Machine Unlearning},'' in \emph{Proc. of IEEE S \& P}, 2015, pp. 463--480.

\bibitem{Bourtoule21}
L.~Bourtoule, V.~Chandrasekaran, C.~A. Choquette-Choo, H.~Jia, A.~Travers, B.~Zhang, D.~Lie, and N.~Papernot, ``{Machine Unlearning},'' in \emph{Proc. of IEEE S \& P}, 2021, pp. 141--159.

\bibitem{Xu23}
H.~Xu, T.~Zhu, L.~Zhang, W.~Zhou, and P.~S. Yu, ``{Machine Unlearning: A Survey},'' \emph{ACM Computing Surveys}, vol.~56, no.~1, pp. 9:1--36, 2023.

\bibitem{Ye25USENIX}
D.~Ye, T.~Zhu, J.~Li, K.~Gao, B.~Liu, L.~Y. Zhang, W.~Zhou, and Y.~Zhang, ``Data duplication: A novel multi-purpose attack paradigm in machine unlearning,'' in \emph{Proc. of USENIX Security}, 2025.

\bibitem{Thudi23}
A.~Thudi, H.~Jia, I.~Shumailov, and N.~Papernot, ``{On the Necessity of Auditable Algorithmic Definitions for Machine Unlearning},'' in \emph{Proc. of USENIX Security}, 2023.

\bibitem{Zhang24ICML}
B.~Zhang, Z.~Chen, C.~Shen, and J.~Li, ``{Verification of Machine Unlearning is Fragile},'' in \emph{Proc. of ICML}, 2024.

\bibitem{Xue25}
\BIBentryALTinterwordspacing
L.~Xue, S.~Hu, W.~Lu, Y.~Shen, D.~Li, P.~Guo, Z.~Zhou, M.~Li, Y.~Zhang, and L.~Y. Zhang, ``Towards reliable forgetting: A survey on machine unlearning verification, challenges, and future directions,'' 2025. [Online]. Available: \url{https://arxiv.org/abs/2506.15115}
\BIBentrySTDinterwordspacing

\bibitem{Naderloui25USENIX}
N.~Naderloui, S.~Yan, B.~Wang, J.~Fu, W.~H. Wang, W.~Liu, and Y.~Hong, ``{Rectifying Privacy and Efficacy Measurements in Machine Unlearning: A New Inference Attack Perspective},'' in \emph{Proc. of USENIX Security}, 2025, pp. 5545--5564.

\bibitem{Hayes25SaTML}
J.~Hayes, I.~Shumailov, E.~Triantafillou, A.~Khalifa, and N.~Papernot, ``{Inexact Unlearning Needs More Careful Evaluations to Avoid a False Sense of Privacy},'' in \emph{Proc. of IEEE SaTML}, 2025, pp. 497--519.

\bibitem{Carlini22SP}
N.~Carlini, S.~Chien, M.~Nasr, S.~Song, A.~Terzis, and F.~Tramer, ``{Membership Inference Attacks From First Principles},'' in \emph{Proc. of IEEE S \& P}, 2022, pp. 1897--1914.

\bibitem{Zhang25SaTML}
J.~Zhang, D.~Das, G.~Kamath, and F.~Tramer, ``{Membership Inference Attacks Cannot Prove that a Model Was Trained On Your Data},'' in \emph{Proc. of IEEE SaTML}, 2025, pp. 333--345.

\bibitem{Chowdhury25SaTML}
A.~R. Chowdhury, Z.~Kong, and K.~Chaudhuri, ``{On the Reliability of Membership Inference Attacks},'' in \emph{Proc. of IEEE SaTML}, 2025, pp. 534--549.

\bibitem{Wang25CCS}
Z.~Wang, C.~Zhang, Y.~Chen, N.~Baracaldo, S.~Kadhe, and L.~Yu, ``{Membership Inference Attacks as Privacy Tools: Reliability, Disparity and Ensemble},'' in \emph{Proc. of ACM CCS}, 2025, pp. 1724--1738.

\bibitem{Sommer22PETS}
D.~M. Sommer, L.~Song, S.~Wagh, and P.~Mittal, ``{Athena: Probabilistic Verification of Machine Unlearning},'' in \emph{Proc. of PETS}, 2022, p. 268–290.

\bibitem{Zhang24ICMLb}
B.~Zhang, Y.~Dong, T.~Wang, and J.~Li, ``{Towards Certified Unlearning for Deep Neural Networks},'' in \emph{Proc. of ICML}, 2024.

\bibitem{Koloskova25ICML}
A.~Koloskova, Y.~Allouah, A.~Jha, R.~Guerraoui, and S.~Koyejo, ``{Certified Unlearning for Neural Networks},'' in \emph{Proc. of ICML}, 2025.

\bibitem{Eisenhofer25}
T.~Eisenhofer, D.~Riepel, V.~Chandrasekaran, E.~Ghosh, O.~Ohrimenko, and N.~Papernot, ``{Verifiable and Provably Secure Machine Unlearning},'' in \emph{Proc. of IEEE SaTML}, 2025.

\bibitem{Olson18NeurIPS}
M.~Olson, A.~Wyner, and R.~Berk, ``{Modern Neural Networks Generalize on Small Data Sets},'' in \emph{Proc. of NeurIPS}, 2018.

\bibitem{t-SNE}
L.~van~der Maaten and G.~Hinton, ``{Visualizing Data using t-SNE},'' \emph{Journal of Machine Learning Research}, vol.~9, pp. 2579--2605, 2008.

\bibitem{Guo20}
C.~Guo, T.~Goldstein, A.~Hannun, and L.~van~der Maaten, ``{Certified Data Removal from Machine Learning Models},'' in \emph{Proc. of ICML}, 2020.

\bibitem{Graves21AAAI}
L.~Graves, V.~Nagisetty, and V.~Ganesh, ``{Amnesiac Machine Learning},'' in \emph{Proc. of AAAI}, 2021, pp. 11\,516--11\,524.

\bibitem{NIPSUnlearning23}
\BIBentryALTinterwordspacing
T.~organizers of~the NeurIPS Unlearning~Competition, ``{Evaluation for the NeurIPS Machine Unlearning Competition},'' 2023. [Online]. Available: \url{https://unlearning-challenge.github.io/assets/data/Machine_Unlearning_Metric.pdf}
\BIBentrySTDinterwordspacing

\bibitem{Maini24}
P.~Maini, Z.~Feng, A.~Schwarzschild, Z.~C. Lipton, and J.~Z. Kolter, ``{TOFU: A Task of Fictitious Unlearning for LLMs},'' in \emph{Proc. of COLM}, 2024.

\bibitem{Georgiev25ICLR}
K.~Georgiev, R.~Rinberg, S.~M. Park, S.~Garg, A.~Ilyas, A.~Madry, and S.~Neel, ``{Attribute-to-Delete: Machine Unlearning via Datamodel Matching},'' in \emph{Proc. of ICLR}, 2025.

\bibitem{Deshpande18}
A.~Deshpande and Y.~Kalai, ``Proofs of ignorance and applications to 2-message witness hiding,'' \emph{Cryptology ePrint Archive}, 2018.

\bibitem{Kuykendall20TCC}
B.~Kuykendall and M.~Zhandry, ``{Towards Non-interactive Witness Hiding},'' in \emph{Proc. of TCC}, 2020.

\bibitem{Lycklama24USENIX}
H.~Lycklama, A.~Viand, N.~Küchler, C.~Knabenhans, and A.~Hithnawi, ``{Holding Secrets Accountable: Auditing Privacy-Preserving Machine Learning},'' in \emph{Proc. of USENIX Security}, 2024.

\bibitem{Krizhevsky14}
\BIBentryALTinterwordspacing
A.~Krizhevsky, V.~Nair, and G.~Hinton, ``The cifar-10 dataset,'' 2014. [Online]. Available: \url{http://www.cs.toronto.edu/kriz/cifar.html}
\BIBentrySTDinterwordspacing

\bibitem{Netzer11NIPS}
Y.~Netzer, T.~Wang, A.~Coates, A.~Bissacco, B.~Wu, and A.~Y. Ng, ``{Reading Digits in Natural Images with Unsupervised Feature Learning},'' in \emph{Proc. of NIPS Workshop on Deep Learning and Unsupervised Feature Learning}, 2011, pp. 1--9.

\bibitem{SkinCancer}
\BIBentryALTinterwordspacing
M.~H. Javid, ``Melanoma skin cancer dataset,'' 2022. [Online]. Available: \url{https://www.kaggle.com/dsv/3376422}
\BIBentrySTDinterwordspacing

\bibitem{BBCNews}
\BIBentryALTinterwordspacing
D.~Gritten, ``{Israel targets Hamas’s labyrinth of tunnels under Gaza},'' 2023. [Online]. Available: \url{https://www.bbc.com/news/world-middle-east-67097124}
\BIBentrySTDinterwordspacing

\bibitem{AGNews}
X.~Zhang, J.~Zhao, and Y.~LeCun, ``{Character-level Convolutional Networks for Text Classification},'' in \emph{Proc. of NIPS}, 2015, p. 649–657.

\bibitem{IMDB}
\BIBentryALTinterwordspacing
A.~Anand, ``Imdb.'' [Online]. Available: \url{https://www.kaggle.com/datasets/atulanandjha/imdb-50k-movie-reviews-test-your-bert}
\BIBentrySTDinterwordspacing

\bibitem{Thudi22EuroSP}
A.~Thudi, G.~Deza, V.~Chandrasekaran, and N.~Papernot, ``{Unrolling SGD: Understanding Factors Influencing Machine Unlearning},'' in \emph{Proc. of EuroS\&P}, 2022, pp. 303--319.

\bibitem{Golatkar20CVPR}
A.~Golatkar, A.~Achille, and S.~Soatto, ``{Eternal sunshine of the spotless net: Selective forgetting in deep networks},'' in \emph{Proc. of CVPR}, 2020, p. 9304–9312.

\bibitem{Golatkar20ECCV}
------, ``{Forgetting Outside the Box: Scrubbing Deep Networks of Information Accessible from Input-Output Observations},'' in \emph{Proc. of ECCV}, 2020, pp. 383--398.

\bibitem{Yao24NeurIPS}
Y.~Yao, X.~Xu, and Y.~Liu, ``{Large Language Model Unlearning},'' in \emph{Proc. of NeurIPS}, 2024.

\bibitem{OPT2.7B}
\BIBentryALTinterwordspacing
S.~Zhang and et~al., ``{OPT: Open Pre-trained Transformer Language Models},'' 2022. [Online]. Available: \url{https://arxiv.org/abs/2205.01068}
\BIBentrySTDinterwordspacing

\bibitem{Ji23NeurIPS}
J.~Ji and et~al., ``{BeaverTails: Towards Improved Safety Alignment of LLM via a Human-Preference Dataset},'' in \emph{Proc. of NeurIPS}, 2023.

\bibitem{Lin22ACL}
S.~Lin, J.~Hilton, and O.~Evans, ``{TruthfulQA: Measuring How Models Mimic Human Falsehoods},'' in \emph{Proc. of ACL}, 2022, pp. 3214--3252.

\bibitem{Zhang20ICLR}
T.~Zhang, V.~Kishore, F.~Wu, K.~Q. Weinberger, and Y.~Artzi, ``{BERTScore: Evaluating Text Generation with BERT},'' in \emph{Proc. of ICLR}, 2020.

\bibitem{Chen22CCS}
M.~Chen, Z.~Zhang, T.~Wang, M.~Backes, M.~Humbert, and Y.~Zhang, ``{Graph Unlearning},'' in \emph{Proc. of CCS}, 2022.

\bibitem{Golatkar21CVPR}
A.~Golatkar, A.~Achille, A.~Ravichandran, M.~Polito, and S.~Soatto, ``{Mixed-Privacy Forgetting in Deep Networks},'' in \emph{Proc. of CVPR}, 2021, pp. 792--801.

\bibitem{Li25WWW}
Z.~Li, Q.~Ye, and H.~Hu, ``{FUNU: Boosting Machine Unlearning Efficiency by Filtering Unnecessary Unlearning},'' in \emph{Proc. of The ACM Web Conference}, 2025.

\bibitem{Liu24NatureMI}
S.~Liu and et~al., ``{Rethinking Machine Unlearning for Large Language Models},'' \emph{Nature Machine Intelligence}, 2024.

\bibitem{Wang25ICLR}
Q.~Wang, J.~Zhou, Z.~Zhou, S.~Shin, B.~Han, and K.~Q. Weinberger, ``{Rethinking LLM Unlearning Objectives: A Gradient Perspective and Go Beyond},'' in \emph{Proc. of ICLR}, 2025.

\bibitem{Chen23EMNLP}
J.~Chen and D.~Yang, ``{Unlearn What You Want to Forget: Efficient Unlearning for LLMs},'' in \emph{Proc. of EMNLP}, 2023, p. 12041–12052.

\bibitem{Pawelczyk24ICML}
M.~Pawelczyk, S.~Neel, and H.~Lakkaraju, ``{In-Context Unlearning: Language Models as Few Shot Unlearners},'' in \emph{Proc. of ICML}, 2024.

\bibitem{Liu24NIPS}
C.~Y. Liu, Y.~Wang, J.~Flanigan, and Y.~Liu, ``{Large Language Model Unlearning via Embedding-Corrupted Prompts},'' in \emph{Proc. of NeurIPS}, 2024.

\bibitem{Sekhari21NIPS}
A.~Sekhari, J.~Acharya, G.~Kamath, and A.~T. Suresh, ``{Remember What You Want to Forget: Algorithms for Machine Unlearning},'' in \emph{Proc. of NeurIPS}, 2021, p. 18075–18086.

\bibitem{Dwork14}
C.~Dwork and A.~Roth, ``{The Algorithmic Foundations of Differential Privacy},'' \emph{Foundations and Trends in Theoretical Computer Science}, vol.~9, no. 3-4, p. 211–407, 2014.

\bibitem{Weng24TIFS}
J.~Weng, S.~Yao, Y.~Du, J.~Huang, J.~Weng, and C.~Wang, ``{Proof of Unlearning: Definitions and Instantiation},'' \emph{IEEE Transactions on Information Forensics and Security}, vol.~19, pp. 3309--3323, 2024.

\bibitem{Shokri17}
R.~Shokri, M.~Stronati, C.~Song, and V.~Shmatikov, ``{Membership Inference Attacks Against Machine Learning Models},'' in \emph{Proc. of IEEE S \& P}, 2017, pp. 3--18.

\bibitem{Attia22CoLT}
A.~Attia and T.~Koren, ``{Uniform Stability for First-Order Empirical Risk Minimization},'' in \emph{Proc. of CoLT}, 2022, pp. 3313--3332.

\bibitem{Dai25ICML}
B.~Dai, ``{EnsLoss: Stochastic Calibrated Loss Ensembles for Preventing Overfitting in Classification},'' in \emph{Proc. of ICML}, 2025.

\bibitem{Wang23CoLT}
Y.~Wang and C.~Scott, ``{On Classification-Calibration of Gamma-Phi Losses},'' in \emph{Proc. of CoLT}, 2023, pp. 4929--4951.

\end{thebibliography}




\section*{Appendix}
\setcounter{section}{0}
\renewcommand{\appendixname}{Appendix~\Alph{section}}


\vspace{-0mm}
\section{Proof of Main Conclusions}

\begin{proof}[Proof of Proposition \ref{prop:preliminary check}]
    We prove the proposition by contradiction. Assume that $M_u$ is successfully unlearned. By the definition of successful unlearning, every successfully unlearned model must be behaviorally separated from the original model on the unlearning set $D_u$ by at least $\gamma$. Therefore, as $M_u$ is assumed to be successfully unlearned, it must satisfy $A_u(M,M_u)\leq 1-\gamma$. On the other hand, the premise of the proposition states that $A_u(M,M_u)\geq 1-\epsilon$, where $\epsilon\in(0,\gamma)$. As $\epsilon<\gamma$, we have $1-\epsilon>1-\gamma$. Hence,
    \[
        A_u(M,M_u)\geq 1-\epsilon>1-\gamma.
    \]
    This contradicts the necessary condition for successful unlearning: $A_u(M,M_u)\leq 1-\gamma$. Therefore, our assumption that $M_u$ is successfully unlearned must be false. This completes the proof.
\end{proof}

\begin{proof}[Proof of Theorem \ref{thm:Polynomial}]
    Each refinement step corresponds to a standard mini-batch stochastic gradient descent (SGD) or empirical risk minimization (ERM) procedure on $D_\Delta$. With batch size $b$ and $E$ training epochs, the total number of gradient evaluations is $E \cdot \lceil |D_\Delta| / b \rceil$. The computational cost of each forward–backward pass is linear, and thus polynomial, in both the number of model parameters and the number of nonzero elements in the input and activation layers. Therefore, the total computational complexity of one refinement step is $O(E\cdot\lceil|D_\Delta/b|\rceil\cdot\mathrm{Poly}(S))$. For fixed $b$ and $E$, this simplifies to $\mathrm{Poly}(|D_\Delta|, S)$, completing the proof.
\end{proof}

\begin{proof}[Proof of Lemma \ref{lem:stability}]
The proof is based on uniform stability. It is a widely accepted smoothness property of ERM and SGD: small perturbations in the training data lead to proportionally small changes in the model’s output distribution \cite{Attia22CoLT}. 

Consider a sequence of datasets $D_1, ..., D_m$, where any two adjacent datasets, $D_k$ and $D_{k+1}$, differ by a single example. By applying the triangle inequality, we have:
\begin{equation}\nonumber
    \mathbb{E}||M-M'||_1\leq\sum^m_{k=1}\mathbb{E}||M_{k+1}-M_{k}||_1.
\end{equation}
According to the uniform stability property,
\begin{equation}\nonumber
    \mathbb{E}|[l(M_{k+1},x)-l(M_{k},x)]|\leq\frac{\beta}{|D_r|},
\end{equation}
where $l$ is a loss function and $\beta$ is a constant determined by the step size and Lipschitz continuity parameters. Since $l$ is assumed to be Lipschitz continuous, we have:
\begin{equation}\nonumber
    \mathbb{E}||M_{k+1}-M_{k}||_1\leq\frac{1}{\lambda}\mathbb{E}|[l(M_{k+1},x)-l(M_{k},x)]|\leq\frac{\beta}{\lambda|D_r|},
\end{equation}
where $\lambda$ is the calibration constant. Summing over the $m$ single replacements yields:
\begin{equation}\nonumber
    \mathbb{E}||M-M'||_1\leq\frac{m\beta}{\lambda|D_r|}=C\cdot\frac{m}{|D_r|},
\end{equation}
where $C$ absorbs the constants $\beta$ and $\lambda$.
\end{proof}

\begin{proof}[Proof of Lemma \ref{lem:reachability}]
Let $D^{(0)}_v$ denote the initial verification set, and define $D^{(t)}_v = D^{(t-1)}_v \cup D_{\Delta_{t-1}}$, where $D_{\Delta_{t-1}} \subset D_r - D^{(t-1)}_v$. Let $m_t = |D_r - D^{(t)}_v|$ denote the number of samples not included at step $t$. As $D^{(t)}_v \subset D_r$, $m_t$ is non-increasing.

According to Lemma \ref{lem:stability}, if two datasets differ by $m$ samples, the prediction discrepancy between their corresponding trained models is bounded by $C \cdot m / |D_r|$. Hence,
\begin{equation}\label{eq:upperBound}
    \mathbb{E}_{x\sim\mathcal{X}}[||M^{(t)}_v(x)-M^*_r(x)||_1]\leq C\cdot\frac{m_t}{|D_r|}.
\end{equation}
To ensure (\ref{eq:upperBound}) $\leq\epsilon$, it suffices that 
\begin{equation}\label{eq:upperBound2}
    m_t\leq\frac{\epsilon}{C}\cdot|D_r|.
\end{equation}

 We now construct a concrete refinement schedule that achieves (\ref{eq:upperBound2}) within a number of steps polynomial in $1/\epsilon$. Choose a constant fraction $\rho \in (0, 1)$ and, at each step, add a $\rho$-fraction of the current missed set: $|D_{\Delta_t}| = \rho m_t$. Then, $m_{t+1} = m_t - |D_{\Delta_t}| = (1 - \rho)m_t$. By induction,
\begin{equation}\label{eq:m induction}
    m_t=(1-\rho)^t m_0,
\end{equation}
where $m_0=|D_r-D^{(0)}_v|$. Combining (\ref{eq:upperBound}) and (\ref{eq:m induction}) yields
\begin{equation}\label{eq:bound induction}
    \mathbb{E}_{x\sim\mathcal{X}}[||M^{(t)}_v(x)-M^*_r(x)||_1]\leq C\cdot\frac{(1-\rho)^t m_0}{|D_r|}\leq C\cdot(1-\rho)^t.
\end{equation}
To make (\ref{eq:bound induction}) $\leq\epsilon$, it suffices to take
\begin{equation}\nonumber
    t\geq\frac{\mathrm{log}(\epsilon/C)}{\mathrm{log}(1-\rho)}.
\end{equation}
Since $C$ and $\rho$ are constants, we have $t = O(\mathrm{log}(1/\epsilon))$.
Hence, with $T = O(\mathrm{log}(1/\epsilon))$ refinement steps, (\ref{eq:upperBound2}) holds, and therefore (\ref{eq:upperBound}) $\leq \epsilon$.

By Theorem \ref{thm:Polynomial}, each refinement step runs in time $\mathrm{Poly}(|D_{\Delta_t}|, S)$. Thus, a logarithmic (and polynomial) number of such steps ensures a polynomial-time procedure.
\end{proof}

\begin{proof}[Proof of Lemma \ref{lem:alignment}]
    The result follows directly from the triangle inequality and the conclusion of Lemma \ref{lem:reachability}:
\begin{equation}\nonumber
\begin{aligned}    
    &\mathbb{E}_{x\sim\mathcal{X}}[||M^{(t)}_v(x)-M_u(x)||_1]\\
    &\leq\mathbb{E}_{x\sim\mathcal{X}}[||M^{(t)}_v(x)-M^*_r(x)||_1]+\mathbb{E}_{x\sim\mathcal{X}}[||M^*_r(x)-M_u(x)||_1]\\
    &\leq\epsilon+\epsilon=2\epsilon.
\end{aligned}
\end{equation}
\end{proof}

\begin{proof}[Proof of Theorem \ref{thm:accept}]
We prove this theorem by connecting it to the concept of proof-of-ignorance.
By setting $\epsilon = \delta / 2$ in Lemma \ref{lem:alignment}, the refinement process of $M_v$ constitutes a polynomial-time reduction that never accesses $D_u$.
This forms the proof of ignorance analogue: a polynomial-time reduction maps the reference solution $s'$ (the parameters of $M^{(0)}_v$) to an intermediate one $\Tilde{s}$ (the parameters of $M^{(t)}_v$) that is indistinguishable from the alleged solution $s$ (the parameters of $M_u$) without ever using the secret, i.e., $D_u$.
\end{proof}

\begin{proof}[Proof of Theorem \ref{thm:reject}]
Let $P$ denote the distribution of $D_u$, and let $M_{D_u}$ represent the model trained with dependence on $D_u$.
Formally, ``dependence'' means that the model’s prediction rule on $P$ changes when $D_u$ is removed.
This implies the existence of a $D_r$-only model $M^*_r$ such that the excess $P$-risk is strictly positive under inputs from $P$: 
\begin{equation}\nonumber
    \mathbb{E}_P[||M^*_r-M_{D_u}||_1]>0.
\end{equation}
Under classification-calibrated losses \cite{Dai25ICML}, the excess risk provides a lower bound on the output-space discrepancy via a calibration function $\psi$, for some constant $c > 0$:
\begin{equation}\nonumber
    \mathbb{E}_P[||M^*_r-M_{D_u}||_1]> c\psi(\mathcal{R}_P(M^*_r)-\mathcal{R}_P(M_{D_u})),
\end{equation}
where $\mathcal{R}_P(M)=\mathbb{E}_{(x,y)\sim P}[l(M(x),y)]$. 
We observe that $M^*_r$ is equivalent to $M^{(t)}_v$ once all polynomial-length refinement sequences have been executed, and $M_{D_u}$ corresponds to $M_u$ when $\mathbb{E}[||M^{(t)}_v - M_u||_1] > \delta$.
By setting $\delta = c \psi(\mathcal{R}_P(M^*_r) - \mathcal{R}_P(M_{D_u}))$, we conclude the proof.
\end{proof}

\begin{proof}[Proof of Corollary \ref{cor:sufficiency}]
Based on the calibration theory of loss functions \cite{Wang23CoLT}, there exists a monotone calibration function $\phi$ s. t.
\begin{equation}\nonumber
    \phi(Acc(M_u,D_u)-Acc(M^{(0)}_v,D_u))\leq\mathbb{E}[||M_u-M^{(0)}_v||_1].
\end{equation}
When the model outputs are well-calibrated probability distributions, such as softmax outputs, this relationship can be linearized for small perturbations, yielding
\begin{equation}\nonumber
    \lambda\Delta_{acc}\leq\mathbb{E}[||M_u-M^{(0)}_v||_1].
\end{equation}
From (\ref{eq:upperBound}) and (\ref{eq:m induction}) in Lemma \ref{lem:reachability}, we obtain
\begin{equation}\nonumber
\begin{aligned}
    &\lambda\Delta_{acc}\leq\mathbb{E}[||M_u-M^{(0)}_v||_1]\leq C\cdot\frac{m_0}{|D_r|}\\
    \Rightarrow &m_0\geq\frac{\lambda}{C}\Delta_{acc}|D_r|\Rightarrow\frac{m_1}{1-\rho}\geq\frac{\lambda}{C}\Delta_{acc}|D_r|\\
    \Rightarrow&m_1\geq\frac{\lambda}{C}(1-\rho)\Delta_{acc}|D_r|.
\end{aligned}
\end{equation}
By Lemma \ref{lem:alignment} and Theorem \ref{thm:accept}, setting $m_1 \geq \frac{\lambda}{C}(1 - \rho)\Delta{acc}|D_r|$ ensures that the gap between $M_u$ and $M^{(1)}_v$ falls below $\delta$, thus completing the proof.
\end{proof}

\section{Datasets and Unlearning Methods}
\noindent\textbf{Datasets.}
\begin{itemize}[leftmargin=*]
    \item \textbf{CIFAR10} \cite{Krizhevsky14} includes $60,000$ images across $10$ classes, each containing $6,000$ images of vehicles and animals. The dimension of each image is $32\times 32$. 
    \item \textbf{SVHN} \cite{Netzer11NIPS} is a street view house number dataset, consisting of $10$ classes, each representing a digit. It comprises over $600,000$ samples, each measuring $32\times 32$.
    \item \textbf{SkinCancer} \cite{SkinCancer} is a melanoma skin cancer dataset containing $10,000$ images, with $9,000$ images for training and $1,000$ for testing. The task is binary classification, and each image was resized to $32\times 32$. 
    \item \textbf{BBCNews} \cite{BBCNews} consists of RSS Feeds from the BBC News site, comprising $29,500$ records. Each record includes five attributes: title, pubDate, guid, link, and description. The dataset includes five classes: business, politics, entertainment, science, and sport.
    \item \textbf{AGNews} \cite{AGNews}  is a topic classification dataset composed of news articles categorized into four classes: World, Sports, Business, and Technology. Each class includes $30,000$ training samples and $1,900$ test samples.
    \item \textbf{IMDB} \cite{IMDB} contains 50,000 movie reviews categorized for binary sentiment classification, distinguishing between positive and negative sentiments. The dataset comprises 25,000 highly polar movie reviews allocated for training and an additional 25,000 for testing purposes.
\end{itemize}

\vspace{1mm}
\noindent\textbf{Unlearning Methods.}
\begin{itemize}[leftmargin=*]
    \item \textbf{Retrain.} This method removes the unlearning data and retrains a new model on the remaining data. Variant includes SISA \cite{Bourtoule21}, which trains models on data shards and retrains affected shards upon unlearning requests. 

    \item \textbf{Adversarial retrain \cite{Zhang24ICML}.} This method retrains a model from scratch without using unlearning samples. During retraining, each sample in the mini-batch is examined: if it belongs to the remaining data, it is retained for model updates; if it is an unlearning sample, the method selects a similar sample from the remaining data that produces a comparable gradient. 


    \item \textbf{Fine-tuning (FT) \cite{Koloskova25ICML}.} This method directly fine-tunes the current model on the remaining data. Its variants \cite{Koloskova25ICML} apply noisy fine-tuning on all or part of the remaining data. 

    \item \textbf{Relabeling \cite{Graves21AAAI}.} This family of methods modifies the unlearned data by assigning randomly selected incorrect labels, and uses the mislabeled data to fine-tune the model. 
    

    \item \textbf{Relabeling with fine-tuning (Relabel+FT) \cite{Graves21AAAI}.} This approach builds on the relabeling method. After unlearning, the unlearned model is further fine-tuned using the remaining data to improve its overall performance. 

    \item \textbf{Gradient ascent (GA) \cite{Thudi22EuroSP}.} This class of methods reverses the model's training on the unlearned data by adding back the corresponding gradients. 

    \item \textbf{Gradient ascent with fine-tuning (GA+FT) \cite{Thudi22EuroSP}.} This approach fine-tunes the unlearned model on the remaining data after applying the gradient ascent unlearning method. 

    \item \textbf{Fisher forgetting \cite{Golatkar20CVPR}.} This method applies Gaussian noise to perturb the model towards exact unlearning. The Gaussian distribution has a zero mean and covariance determined by the $4$-th root of the Fisher information matrix with respect to the model on the unlearned data.

    \item \textbf{Hessian forgetting \cite{Golatkar20CVPR,Golatkar20ECCV}.} This method builds upon Fisher forgetting by replacing the Fisher Information Matrix with the Hessian matrix, which provides a second-order correction to achieve more precise unlearning. 
    
    \item \textbf{Certified Hessian forgetting \cite{Zhang24ICMLb}.} This method offers a theoretical guarantee of certified unlearning by modifying the model parameters using the inverse Hessian matrix of the loss function computed over the remaining data. 

    
\end{itemize}

\section{Additional Experimental Results}
\subsection{Adaptability Study}
\noindent\textbf{Varying the Size of the Unlearning Set $D_u$.} The results discussed above were obtained using $20\%$ of the training data as $D_u$. To assess the adaptability of our auditing method, we evaluate its performance using $30\%$ of the training data as $D_u$. The auditing results are shown in Tables~\ref{tab:CIFAR10Du=0.3}, \ref{tab:SVHNDu=0.3}, \ref{tab:SkinCancerDu=0.3}, \ref{tab:BBCNewsDu=0.3}, \ref{tab:AGNewsDu=0.3}, and \ref{tab:IMDBDu=0.3}.


We find that the auditing results remain highly consistent across different values of $|D_u|$, which demonstrates adaptability of our proposed auditing method. A key observation is that as the size of $D_u$ increases, the classification accuracy on both the unlearning set $D_u$ and the test set $D_t$ generally decreases after unlearning. This decline can be attributed to the fact that removing a larger portion of the training data naturally reduces the generalizability of the resulting model. Notably, when unlearning is successful, $D_u$ becomes effectively unseen to the model, similar to $D_t$, and thus a drop in accuracy on both sets is expected.


\begin{table}[!ht]\scriptsize
\vspace{-1mm}
	\centering
	\caption{Auditing results of our method on CIFAR10 with $|D_u|=0.3\cdot|D|$.}
\begin{tabular} {cccccc} 
\toprule
 \multirow{2}*{\makecell[c]{Unlearning\\Approaches}} & \multirow{2}*{\makecell[c]{Agree. on $D_u$\\(Base: $65.77$)}} & \multirow{2}*{\makecell[c]{KL on $D_u$\\(Base: $2.45$)}} & \multicolumn{3}{c}{Accuracy on} \\\cline{4-6}
 & & & $D_u$ & $D_r$ & $D_t$ \\
 \midrule
Pre-unlearn &  &  & $95.92$ & $95.61$ & $71.79$ \\\hline
Retrain$@0.3$ & $66.17$ & $2.09$ & $69.61$ & $99.94$ & $69.90$ \\ 
Adv. Retr.$@0.3$ & $66.01$ & $2.19$ & $70.23$ & $99.91$ & $69.67$ \\
Fine-tune$@0.3$ & $66.81$ & $2.13$ & $75.15$ & $99.95$ & $72.54$ \\
Relabel$@0.1$ & $37.61$ & $7.27$ & $54.02$ & $54.41$ & $44.13$ \\
Relab.+FT$@0.3$ & $69.97$ & $1.93$ & $75.39$ & $99.97$ & $72.16$ \\
GA$@0.1$ & $44.15$ & $5.31$ & $58.27$ & $58.04$ & $47.67$ \\
GA+FT$@0.3$ & $69.87$ & $1.92$ & $75.25$ & $99.96$ & $72.36$ \\
Fisher$@0.3$ & $69.81$ & $1.98$ & $99.42$ & $99.42$ & $72.33$ \\
Hessian$@0.3$ & $69.89$ & $1.98$ & $99.59$ & $99.68$ & $73.03$ \\
Cert.-Hess.$@0.3$ & $69.92$ & $1.98$ & $99.95$ & $99.96$ & $73.54$ \\
\bottomrule
\end{tabular}
	\label{tab:CIFAR10Du=0.3}
 \vspace{-1mm}
\end{table}

\begin{table}[!ht]\scriptsize
\vspace{-1mm}
	\centering
	\caption{Auditing results of our method on SVHN with $|D_u|=0.3\cdot|D|$.}
\begin{tabular} {cccccc} 
\toprule
 \multirow{2}*{\makecell[c]{Unlearning\\Approaches}} & \multirow{2}*{\makecell[c]{Agree. on $D_u$\\(Base: $88.87$)}} & \multirow{2}*{\makecell[c]{KL on $D_u$\\(Base: $0.53$)}} & \multicolumn{3}{c}{Accuracy on} \\\cline{4-6}
 & & & $D_u$ & $D_r$ & $D_t$ \\
 \midrule
Pre-unlearn &  &  & $99.49$ & $99.29$ & $89.94$ \\\hline
Retrain$@0.3$ & $89.30$ & $0.48$ & $90.42$ & $99.78$ & $89.60$ \\ 
Adv. Retr.$@0.3$ & $89.08$ & $0.49$ & $90.54$ & $99.81$ & $89.57$ \\
Fine-tune$@0.3$ & $88.88$ & $0.50$ & $91.42$ & $99.91$ & $90.19$ \\
Relabel$@0.05$ & $51.31$ & $8.16$ & $56.77$ & $56.46$ & $48.95$ \\
Relab.+FT$@0.3$ & $89.44$ & $0.48$ & $91.29$ & $99.66$ & $90.58$ \\
GA$@0.05$ & $54.51$ & $5.80$ & $58.05$ & $58.46$ & $51.45$ \\
GA+FT$@0.3$ & $88.95$ & $0.74$ & $89.68$ & $92.99$ & $89.64$ \\
Fisher$@0.3$ & $89.43$ & $0.53$ & $99.42$ & $99.47$ & $90.07$ \\
Hessian$@0.3$ & $89.52$ & $0.52$ & $99.53$ & $99.45$ & $90.11$ \\
Cert.-Hess.$@0.3$ & $89.53$ & $0.52$ & $99.76$ & $99.76$ & $90.44$ \\
\bottomrule
\end{tabular}
	\label{tab:SVHNDu=0.3}
 \vspace{-1mm}
\end{table}

\begin{table}[!ht]\scriptsize
\vspace{-1mm}
	\centering
	\caption{Auditing results of our method on SkinCancer with $|D_u|=0.3\cdot|D|$.}
\begin{tabular} {cccccc} 
\toprule
 \multirow{2}*{\makecell[c]{Unlearning\\Approaches}} & \multirow{2}*{\makecell[c]{Agree. on $D_u$\\(Base: $88.63$)}} & \multirow{2}*{\makecell[c]{KL on $D_u$\\(Base: $0.62$)}} & \multicolumn{3}{c}{Accuracy on} \\\cline{4-6}
 & & & $D_u$ & $D_r$ & $D_t$ \\
 \midrule
Pre-unlearn &  &  & $99.74$ & $99.77$ & $91.70$ \\\hline
Retrain$@0.3$ & $91.43$ & $0.54$ & $91.05$ & $99.96$ & $90.60$ \\ 
Adv. Retr.$@0.2$ & $89.59$ & $0.61$ & $90.32$ & $99.99$ & $90.50$ \\
Fine-tune$@0.3$ & $90.74$ & $0.63$ & $91.33$ & $99.99$ & $91.20$ \\
Relabel$@0.05$ & $45.25$ & $4.83$ & $53.09$ & $54.16$ & $47.90$ \\
Relab.+FT$@0.2$ & $89.45$ & $0.79$ & $90.49$ & $99.87$ & $91.50$ \\
GA$@0.05$ & $55.66$ & $3.92$ & $55.17$ & $54.57$ & $57.80$ \\
GA+FT$@0.2$ & $89.52$ & $0.65$ & $90.87$ & $99.93$ & $90.40$ \\
Fisher$@0.3$ & $91.71$ & $0.54$ & $99.79$ & $99.93$ & $91.30$ \\
Hessian$@0.3$ & $91.67$ & $0.53$ & $99.83$ & $99.82$ & $91.20$ \\
Cert.-Hess.$@0.3$ & $91.64$ & $0.54$ & $100$ & $100$ & $91.30$ \\
\bottomrule
\end{tabular}
	\label{tab:SkinCancerDu=0.3}
 \vspace{-1mm}
\end{table}

\begin{table}[!ht]\scriptsize
	\centering
	\caption{Auditing results of our method on BBCNews with $|D_u|=0.3\cdot|D|$.}
\begin{tabular} {cccccc} 
\toprule
 \multirow{2}*{\makecell[c]{Unlearning\\Approaches}} & \multirow{2}*{\makecell[c]{Agree. on $D_u$\\(Base: $90.13$)}} & \multirow{2}*{\makecell[c]{KL on $D_u$\\(Base: $0.16$)}} & \multicolumn{3}{c}{Accuracy on} \\\cline{4-6}
 & & & $D_u$ & $D_r$ & $D_t$ \\
 \midrule
Pre-unlearn &  &  & $99.40$ & $99.34$ & $90.03$ \\\hline
Retrain$@0.4$ & $92.28$ & $0.13$ & $89.60$ & $99.54$ & $89.36$ \\ 
Adv. Retr.$@0.4$ & $91.02$ & $0.21$ & $88.72$ & $99.54$ & $88.62$ \\
Fine-tune$@0.4$ & $90.15$ & $0.26$ & $90.05$ & $99.54$ & $88.64$ \\
Relabel$@0.05$ & $43.38$ & $0.68$ & $60.38$ & $69.74$ & $60.54$ \\
Relab.+FT$@0.4$ & $90.37$ & $0.23$ & $90.18$ & $99.54$ & $88.56$ \\
GA$@0.05$ & $40.23$ & $1.38$ & $54.20$ & $55.44$ & $50.92$ \\
GA+FT$@0.4$ & $90.93$ & $0.25$ & $89.95$ & $99.54$ & $88.74$ \\
Fisher$@0.4$ & $89.13$ & $0.32$ & $99.50$ & $99.29$ & $90.40$ \\
Hessian$@0.05$ & $54.08$ & $1.02$ & $65.17$ & $64.93$ & $62.82$ \\
Cert.-Hess.$@0.4$ & $89.08$ & $0.34$ & $99.48$ & $99.32$ & $90.12$ \\
\bottomrule
\end{tabular}
	\label{tab:BBCNewsDu=0.3}
 \vspace{-1mm}
\end{table}

\begin{table}[!ht]\scriptsize
	\centering
	\caption{Auditing results of our method on AGNews with $|D_u|=0.3\cdot|D|$.}
\begin{tabular} {cccccc} 
\toprule
 \multirow{2}*{\makecell[c]{Unlearning\\Approaches}} & \multirow{2}*{\makecell[c]{Agree. on $D_u$\\(Base: $88.41$)}} & \multirow{2}*{\makecell[c]{KL on $D_u$\\(Base: $0.48$)}} & \multicolumn{3}{c}{Accuracy on} \\\cline{4-6}
 & & & $D_u$ & $D_r$ & $D_t$ \\
 \midrule
Pre-unlearn &  &  & $99.89$ & $99.85$ & $89.90$ \\\hline
Retrain$@0.35$ & $91.00$ & $0.29$ & $89.30$ & $99.94$ & $89.37$ \\ 
Adv. Retr.$@0.35$ & $90.58$ & $0.33$ & $89.09$ & $99.94$ & $88.88$ \\
Fine-tune$@0.3$ & $89.81$ & $0.34$ & $90.53$ & $99.96$ & $89.42$ \\
Relabel$@0.05$ & $54.67$ & $2.20$ & $66.73$ & $69.63$ & $61.04$ \\
Relab.+FT$@0.35$ & $90.82$ & $0.31$ & $90.55$ & $99.96$ & $89.46$ \\
GA$@0.05$ & $61.21$ & $1.37$ & $68.70$ & $69.74$ & $65.81$ \\
GA+FT$@0.35$ & $90.59$ & $0.32$ & $90.56$ & $99.97$ & $89.49$ \\
Fisher$@0.35$ & $88.74$ & $0.49$ & $99.83$ & $99.80$ & $90.18$ \\
Hessian$@0.05$ & $57.70$ & $0.48$ & $64.26$ & $64.55$ & $61.60$ \\
Cert.-Hess.$@0.35$ & $88.07$ & $0.52$ & $99.88$ & $99.89$ & $89.94$ \\
\bottomrule
\end{tabular}
	\label{tab:AGNewsDu=0.3}
 \vspace{-1mm}
\end{table}

\begin{table}[!ht]\scriptsize
	\centering
	\caption{Auditing results of our method on IMDB with $|D_u|=0.3\cdot|D|$.}
\begin{tabular} {cccccc} 
\toprule
 \multirow{2}*{\makecell[c]{Unlearning\\Approaches}} & \multirow{2}*{\makecell[c]{Agree. on $D_u$\\(Base: $78.38$)}} & \multirow{2}*{\makecell[c]{KL on $D_u$\\(Base: $0.52$)}} & \multicolumn{3}{c}{Accuracy on} \\\cline{4-6}
 & & & $D_u$ & $D_r$ & $D_t$ \\
 \midrule
Pre-unlearn &  &  & $99.95$ & $99.99$ & $79.90$ \\\hline
Retrain$@0.35$ & $78.95$ & $0.38$ & $77.22$ & $99.20$ & $78.30$ \\ 
Adv. Retr.$@0.4$ & $79.20$ & $0.44$ & $77.13$ & $99.98$ & $77.34$ \\
Fine-tune$@0.4$ & $80.90$ & $0.41$ & $80.48$ & $100$ & $77.52$ \\
Relabel$@0.05$ & $49.43$ & $0.15$ & $60.32$ & $68.30$ & $59.32$ \\
Relab.+FT$@0.4$ & $80.42$ & $0.41$ & $79.60$ & $99.99$ & $77.80$ \\
GA$@0.05$ & $52.38$ & $0.69$ & $67.30$ & $67.85$ & $62.96$ \\
GA+FT$@0.4$ & $81.05$ & $0.39$ & $80.73$ & $100.0$ & $77.84$ \\%
Fisher$@0.4$ & $79.95$ & $0.58$ & $100.00$ & $99.98$ & $80.20$ \\
Hessian$@0.1$ & $56.45$ & $0.75$ & $70.13$ & $68.56$ & $65.06$ \\
Cert.-Hess.$@0.4$ & $79.95$ & $0.58$ & $100$ & $99.99$ & $80.22$ \\
\bottomrule
\end{tabular}
	\label{tab:IMDBDu=0.3}
 \vspace{-1mm}
\end{table}


\vspace{0mm}
\noindent\textbf{Varying the Distribution of the Unlearning Set $D_u$.} The results discussed above were obtained by uniformly sampling data points from the training set to construct the unlearning set $D_u$. 
To further evaluate the adaptability of our auditing method, we now examine its performance under an unbalanced unlearning set.
To create this unbalanced set, we sample $50\%$ of $D_u$ from a single class, while the remaining $50\%$ is uniformly drawn from the other classes. For binary classification tasks such as SkinCancer and IMDB, we sample $75\%$ of $D_u$ from one class and the remaining $25\%$ from the other. Additionally, instead of using uniform random sampling to obtain the verification set $D_v$, we employ stratified random sampling to ensure that the class distribution of $D_v$ mirrors that of $D_u$. 
The auditing results are presented in Tables~\ref{tab:CIFAR10Unbalance}, \ref{tab:SVHNUnbalance}, \ref{tab:SkinCancerUnbalance}, \ref{tab:BBCNewsUnbalance}, \ref{tab:AGNewsUnbalance}, and \ref{tab:IMDBUnbalance}.

It can be observed that our auditing method produces consistent results even when applied to unbalanced $D_u$, demonstrating its strong adaptability across different data distributions. The primary difference between the balanced and unbalanced settings lies in the baseline agreement and KL divergence obtained from comparing the two small verification models, $M_{v_1}$ and $M_{v_2}$. Specifically, when using a balanced $D_u$, the verification models tend to exhibit higher agreement and lower KL divergence compared to the unbalanced case.
This discrepancy may be attributed to the more uniform class distribution in the balanced setting, which allows the verification models to learn more stable decision boundaries. In contrast, the class imbalance in the unbalanced $D_u$ may introduce bias into the verification models, resulting in less consistent behavior and slightly degraded baseline metrics. However, the relative performance trends across different unlearning methods remain consistent, showing the reliability of our auditing framework.

\begin{table}[H]\scriptsize
\vspace{-0mm}
	\centering
	\caption{Auditing results of our method on CIFAR10 with unbalanced $D_u$.}
\begin{tabular} {cccccc} 
\toprule
 \multirow{2}*{\makecell[c]{Unlearning\\Approaches}} & \multirow{2}*{\makecell[c]{Agree. on $D_u$\\(Base: $62.12$)}} & \multirow{2}*{\makecell[c]{KL on $D_u$\\(Base: $3.15$)}} & \multicolumn{3}{c}{Accuracy on} \\\cline{4-6}
 & & & $D_u$ & $D_r$ & $D_t$ \\
 \midrule
Pre-unlearn &  &  & $99.35$ & $98.94$ & $74.47$ \\\hline
Retrain$@0.2$ & $63.59$ & $3.26$ & $64.13$ & $98.98$ & $70.46$ \\ 
Adv. Retr.$@0.2$ & $64.47$ & $2.82$ & $70.85$ & $99.15$ & $71.66$ \\
Fine-tune$@0.2$ & $63.73$ & $2.82$ & $72.72$ & $99.96$ & $72.76$ \\
Relabel$@0.05$ & $30.95$ & $7.34$ & $51.37$ & $51.86$ & $40.88$ \\
Relab.+FT$@0.2$ & $62.67$ & $2.89$ & $70.35$ & $99.94$ & $72.53$ \\
GA$@0.05$ & $31.81$ & $7.73$ & $47.71$ & $48.81$ & $40.63$ \\
GA+FT$@0.2$ & $64.10$ & $2.80$ & $73.02$ & $99.97$ & $73.15$ \\
Fisher$@0.1$ & $62.76$ & $2.52$ & $98.75$ & $98.08$ & $73.91$ \\
Hessian$@0.1$ & $62.82$ & $2.53$ & $98.48$ & $98.40$ & $74.32$ \\
Cert.-Hess.$@0.1$ & $62.92$ & $2.51$ & $99.57$ & $99.55$ & $75.33$ \\
\bottomrule
\end{tabular}
	\label{tab:CIFAR10Unbalance}
 \vspace{-0mm}
\end{table}

\begin{table}[H]\scriptsize
\vspace{-0mm}
	\centering
	\caption{Auditing results of our method on SVHN with unbalanced $D_u$.}
\begin{tabular} {cccccc} 
\toprule
 \multirow{2}*{\makecell[c]{Unlearning\\Approaches}} & \multirow{2}*{\makecell[c]{Agree. on $D_u$\\(Base: $88.85$)}} & \multirow{2}*{\makecell[c]{KL on $D_u$\\(Base: $0.58$)}} & \multicolumn{3}{c}{Accuracy on} \\\cline{4-6}
 & & & $D_u$ & $D_r$ & $D_t$ \\
 \midrule
Pre-unlearn &  &  & $97.87$ & $97.08$ & $90.80$ \\\hline
Retrain$@0.2$ & $89.81$ & $0.74$ & $91.80$ & $96.06$ & $89.01$ \\ 
Adv. Retr.$@0.2$ & $89.98$ & $0.70$ & $92.30$ & $96.92$ & $89.69$ \\
Fine-tune$@0.2$ & $89.70$ & $0.63$ & $92.54$ & $99.65$ & $90.17$ \\
Relabel$@0.05$ & $44.13$ & $10.27$ & $46.30$ & $55.39$ & $47.98$ \\
Relab.+FT$@0.2$ & $89.50$ & $0.64$ & $92.46$ & $99.74$ & $90.06$ \\
GA$@0.05$ & $32.73$ & $10.49$ & $34.79$ & $57.72$ & $51.03$ \\
GA+FT$@0.2$ & $89.62$ & $0.63$ & $92.48$ & $99.67$ & $90.05$ \\
Fisher$@0.2$ & $90.87$ & $0.55$ & $97.47$ & $96.76$ & $90.70$ \\
Hessian$@0.2$ & $90.84$ & $0.55$ & $97.43$ & $96.79$ & $90.63$ \\
Cert.-Hess.$@0.2$ & $90.99$ & $0.53$ & $97.97$ & $97.65$ & $90.94$ \\
\bottomrule
\end{tabular}
	\label{tab:SVHNUnbalance}
 \vspace{-0mm}
\end{table}

\begin{table}[H]\scriptsize
\vspace{-1mm}
	\centering
	\caption{Auditing results of our method on SkinCancer with unbalanced $D_u$.}
\begin{tabular} {cccccc} 
\toprule
 \multirow{2}*{\makecell[c]{Unlearning\\Approaches}} & \multirow{2}*{\makecell[c]{Agree. on $D_u$\\(Base: $87.19$)}} & \multirow{2}*{\makecell[c]{KL on $D_u$\\(Base: $0.86$)}} & \multicolumn{3}{c}{Accuracy on} \\\cline{4-6}
 & & & $D_u$ & $D_r$ & $D_t$ \\
 \midrule
Pre-unlearn &  &  & $99.90$ & $99.78$ & $90.30$ \\\hline
Retrain$@0.1$ & $87.40$ & $0.83$ & $87.85$ & $98.36$ & $87.60$ \\ 
Adv. Retr.$@0.1$ & $90.31$ & $0.60$ & $90.02$ & $99.62$ & $90.01$ \\
Fine-tune$@0.1$ & $92.03$ & $0.43$ & $94.32$ & $99.60$ & $91.10$ \\
Relabel$@0.05$ & $44.53$ & $6.51$ & $42.76$ & $71.55$ & $65.80$ \\
Relab.+FT$@0.1$ & $91.51$ & $0.56$ & $89.43$ & $96.68$ & $90.30$ \\
GA$@0.05$ & $76.93$ & $2.03$ & $82.24$ & $46.17$ & $54.5$ \\
GA+FT$@0.1$ & $87.33$ & $0.97$ & $79.48$ & $88.33$ & $85.80$ \\
Fisher$@0.1$ & $90.52$ & $0.58$ & $99.79$ & $99.65$ & $90.40$ \\
Hessian$@0.1$ & $90.47$ & $0.58$ & $99.84$ & $99.73$ & $90.42$ \\
Cert.-Hess.$@0.1$ & $90.52$ & $0.58$ & $99.90$ & $99.99$ & $91.20$ \\
\bottomrule
\end{tabular}
	\label{tab:SkinCancerUnbalance}
 \vspace{-0mm}
\end{table}

\begin{table}[ht]\scriptsize
\vspace{-1mm}
	\centering
	\caption{Auditing results of our method on BBCNews with unbalanced $D_u$.}
\begin{tabular} {cccccc} 
\toprule
 \multirow{2}*{\makecell[c]{Unlearning\\Approaches}} & \multirow{2}*{\makecell[c]{Agree. on $D_u$\\(Base: $87.92$)}} & \multirow{2}*{\makecell[c]{KL on $D_u$\\(Base: $0.19$)}} & \multicolumn{3}{c}{Accuracy on} \\\cline{4-6}
 & & & $D_u$ & $D_r$ & $D_t$ \\
 \midrule
Pre-unlearn &  &  & $99.42$ & $99.36$ & $89.82$ \\\hline
Retrain$@0.4$ & $90.28$ & $0.18$ & $85.58$ & $99.47$ & $88.98$ \\ 
Adv. Retr.$@0.4$ & $88.30$ & $0.31$ & $83.90$ & $99.48$ & $88.50$ \\
Fine-tune$@0.4$ & $88.20$ & $0.31$ & $85.65$ & $99.48$ & $88.06$ \\
Relabel$@0.1$ & $42.77$ & $1.40$ & $58.35$ & $79.09$ & $66.56$ \\
Relab.+FT$@0.4$ & $88.80$ & $0.29$ & $85.92$ & $99.48$ & $88.10$ \\
GA$@0.1$ & $52.83$ & $1.13$ & $42.33$ & $74.41$ & $63.28$ \\
GA+FT$@0.4$ & $88.90$ & $0.30$ & $85.97$ & $99.48$ & $87.92$ \\
Fisher$@0.4$ & $85.22$ & $0.47$ & $99.52$ & $99.33$ & $90.50$ \\
Hessian$@0.4$ & $73.17$ & $1.26$ & $84.28$ & $72.59$ & $69.86$ \\
Cert.-Hess.$@0.4$ & $85.20$ & $0.49$ & $99.45$ & $99.35$ & $89.84$ \\
\bottomrule
\end{tabular}
	\label{tab:BBCNewsUnbalance}
 \vspace{-0mm}
\end{table}

\begin{table}[ht]\scriptsize
\vspace{-0mm}
	\centering
	\caption{Auditing results of our method on AGNews with unbalanced $D_u$.}
\begin{tabular} {cccccc} 
\toprule
 \multirow{2}*{\makecell[c]{Unlearning\\Approaches}} & \multirow{2}*{\makecell[c]{Agree. on $D_u$\\(Base: $89.34$)}} & \multirow{2}*{\makecell[c]{KL on $D_u$\\(Base: $0.43$)}} & \multicolumn{3}{c}{Accuracy on} \\\cline{4-6}
 & & & $D_u$ & $D_r$ & $D_t$ \\
 \midrule
Pre-unlearn &  &  & $99.83$ & $99.90$ & $89.85$ \\\hline
Retrain$@0.4$ & $91.61$ & $0.25$ & $88.40$ & $99.92$ & $89.56$ \\ 
Adv. Retr.$@0.4$ & $91.15$ & $0.27$ & $88.69$ & $99.94$ & $89.63$ \\
Fine-tune$@0.4$ & $91.01$ & $0.29$ & $89.08$ & $99.95$ & $89.42$ \\
Relabel$@0.05$ & $48.80$ & $2.25$ & $59.40$ & $79.42$ & $66.63$ \\
Relab.+FT$@0.4$ & $91.04$ & $0.28$ & $89.41$ & $99.96$ & $89.51$ \\
GA$@0.05$ & $53.18$ & $1.92$ & $55.47$ & $76.13$ & $67.17$ \\
GA+FT$@0.4$ & $91.04$ & $0.28$ & $89.58$ & $99.95$ & $89.48$ \\
Fisher$@0.4$ & $88.36$ & $0.52$ & $99.82$ & $99.90$ & $89.76$ \\
Hessian$@0.05$ & $60.09$ & $1.56$ & $70.98$ & $57.69$ & $57.40$ \\
Cert.-Hess.$@0.4$ & $88.35$ & $0.52$ & $99.83$ & $99.90$ & $89.79$ \\
\bottomrule
\end{tabular}
	\label{tab:AGNewsUnbalance}
 \vspace{-0mm}
\end{table}

\begin{table}[ht]\scriptsize
\vspace{-0mm}
	\centering
	\caption{Auditing results of our method on IMDB with unbalanced $D_u$.}
\begin{tabular} {cccccc} 
\toprule
 \multirow{2}*{\makecell[c]{Unlearning\\Approaches}} & \multirow{2}*{\makecell[c]{Agree. on $D_u$\\(Base: $79.38$)}} & \multirow{2}*{\makecell[c]{KL on $D_u$\\(Base: $0.67$)}} & \multicolumn{3}{c}{Accuracy on} \\\cline{4-6}
 & & & $D_u$ & $D_r$ & $D_t$ \\
 \midrule
Pre-unlearn &  &  & $100$ & $99.99$ & $80.10$ \\\hline
Retrain$@0.4$ & $86.22$ & $0.24$ & $77.92$ & $100$ & $79.52$ \\ 
Adv. Retr.$@0.4$ & $80.35$ & $0.43$ & $73.65$ & $100$ & $77.46$ \\
Fine-tune$@0.4$ & $79.92$ & $0.49$ & $78.67$ & $100$ & $77.22$ \\
Relabel$@0.1$ & $54.97$ & $0.53$ & $45.32$ & $73.82$ & $61.52$ \\
Relab.+FT$@0.4$ & $80.33$ & $0.46$ & $77.68$ & $100$ & $77.48$ \\
GA$@0.1$ & $66.40$ & $0.64$ & $51.40$ & $72.13$ & $63.30$ \\
GA+FT$@0.4$ & $79.95$ & $0.52$ & $78.55$ & $99.80$ & $76.54$ \\
Fisher$@0.4$ & $76.90$ & $0.72$ & $100$ & $99.99$ & $80.10$ \\
Hessian$@0.2$ & $62.70$ & $0.84$ & $67.92$ & $67.79$ & $65.06$ \\
Cert.-Hess.$@0.4$ & $76.90$ & $0.72$ & $100$ & $99.99$ & $80.10$ \\
\bottomrule
\end{tabular}
	\label{tab:IMDBUnbalance}
 \vspace{-1mm}
\end{table}


\end{document}